\documentclass{article}

\PassOptionsToPackage{numbers}{natbib}

\usepackage[final]{neurips_2021}

\usepackage[utf8]{inputenc} 
\usepackage[T1]{fontenc}    
\usepackage{hyperref}       
\usepackage{url}            
\usepackage{booktabs}       
\usepackage{amsfonts}       
\usepackage{nicefrac}       
\usepackage{microtype}      
\usepackage[dvipsnames]{xcolor}         
\usepackage{amsmath,amssymb,amsthm}
\usepackage{algorithm,algorithmic}
\usepackage{mathtools}
\usepackage{graphicx}
\usepackage{csquotes}
\hypersetup{
colorlinks=True,
linkcolor=blue,
citecolor=blue,
urlcolor=blue}
\usepackage{float}
\usepackage{bm}
\usepackage{makecell}
\usepackage{enumitem}
\usepackage{caption}
\usepackage{adjustbox}
\usepackage{wrapfig, blindtext}
\usepackage{color-edits}
\usepackage{xspace}
\usepackage{subfigure}
\usepackage{mkolar_definitions, color, comment}

\newcommand{\M}{\mathcal{M}}

\renewcommand{\d}[2]{d\left(#1,#2\right)}

\newcommand{\algname}{\texttt{MobILE}}

\newcommand{\cartpoles}{\texttt{Cartpole-v0}}
\newcommand{\cartpole}{\texttt{Cartpole-v1}}

\newcommand{\reacher}{\texttt{Reacher-v2}}
\newcommand{\swimmer}{\texttt{Swimmer-v2}}
\newcommand{\hopper}{\texttt{Hopper-v2}}
\newcommand{\walker}{\texttt{Walker2d-v2}}
\allowdisplaybreaks
\title{\algname: Model-Based Imitation Learning From Observation Alone}

\author{%
  Rahul Kidambi\thanks{Work initiated when RK was a post-doc at Cornell University; work done outside Amazon.
    } \\
  Amazon Search \& AI\\
  Berkeley CA 94704.\\
  \texttt{rk773@cornell.edu} \\
  \And
  Jonathan D. Chang \\
  CS Department, Cornell University \\
  Ithaca NY 14853. \\
  \texttt{jdc396@cornell.edu} \\
  \AND
  Wen Sun \\
  CS Department, Cornell University \\
  Ithaca NY 14853. \\
  \texttt{ws455@cornell.edu}
}

\begin{document}

\maketitle

\begin{abstract}
    This paper studies Imitation Learning from Observations alone (ILFO) where the learner is presented with expert demonstrations that consist only of states visited by an expert (without access to actions taken by the expert). We present a provably efficient model-based framework~\algname~to solve the ILFO problem.~\algname~involves carefully trading off strategic exploration against imitation - this is achieved by integrating the idea of optimism in the face of uncertainty into the distribution matching imitation learning (IL) framework. We provide a unified analysis for~\algname, and demonstrate that ~\algname~enjoys strong performance guarantees for classes of MDP dynamics that satisfy certain well studied notions of structural complexity. We also show that the ILFO problem is {\em strictly harder} than the standard IL problem by presenting an exponential sample complexity separation between IL and ILFO. We complement these theoretical results with experimental simulations on benchmark OpenAI Gym tasks that indicate the efficacy of~\algname. Code for implementing the \algname~framework is available at \url{https://github.com/rahulkidambi/MobILE-NeurIPS2021}.
\end{abstract}

\section{Introduction}
\label{sec:intro}

This paper considers \emph{Imitation Learning from Observation Alone (ILFO)}. In ILFO, the learner is presented with sequences of states encountered by the expert, {\em without} access to the actions taken by the expert, meaning approaches based on a reduction to supervised learning (e.g., Behavior cloning (BC)~\citep{RossB10}, DAgger \citep{DAgger}) are not applicable. ILFO is more general and has potential for applications where the learner and expert have different action spaces, applications like sim-to-real \citep{song2020provably,desai2020imitation} etc.

Recently,~\citep{Sun19FAIL} reduced the ILFO problem to a sequence of one-step distribution matching problems that results in obtaining a non-stationary policy. This approach, however, is sample inefficient for longer horizon tasks since the algorithm does not effectively reuse previously collected samples when solving the current sub-problem. Another line of work considers model-based methods to infer the expert's actions with either an inverse dynamics~\citep{TorabiBCO} or a forward dynamics~\citep{EdwardsILPO} model; these recovered actions are then fed into an IL approach like BC to output the final policy. These works rely on stronger assumptions that are only satisfied for Markov Decision Processes (MDPs) with injective transition dynamics~\citep{ZhuLDZ20}; we return to this in the related works section. 

\begin{wrapfigure}{r}{0.45\textwidth}
  \begin{center}
    \includegraphics[width=0.45\textwidth]{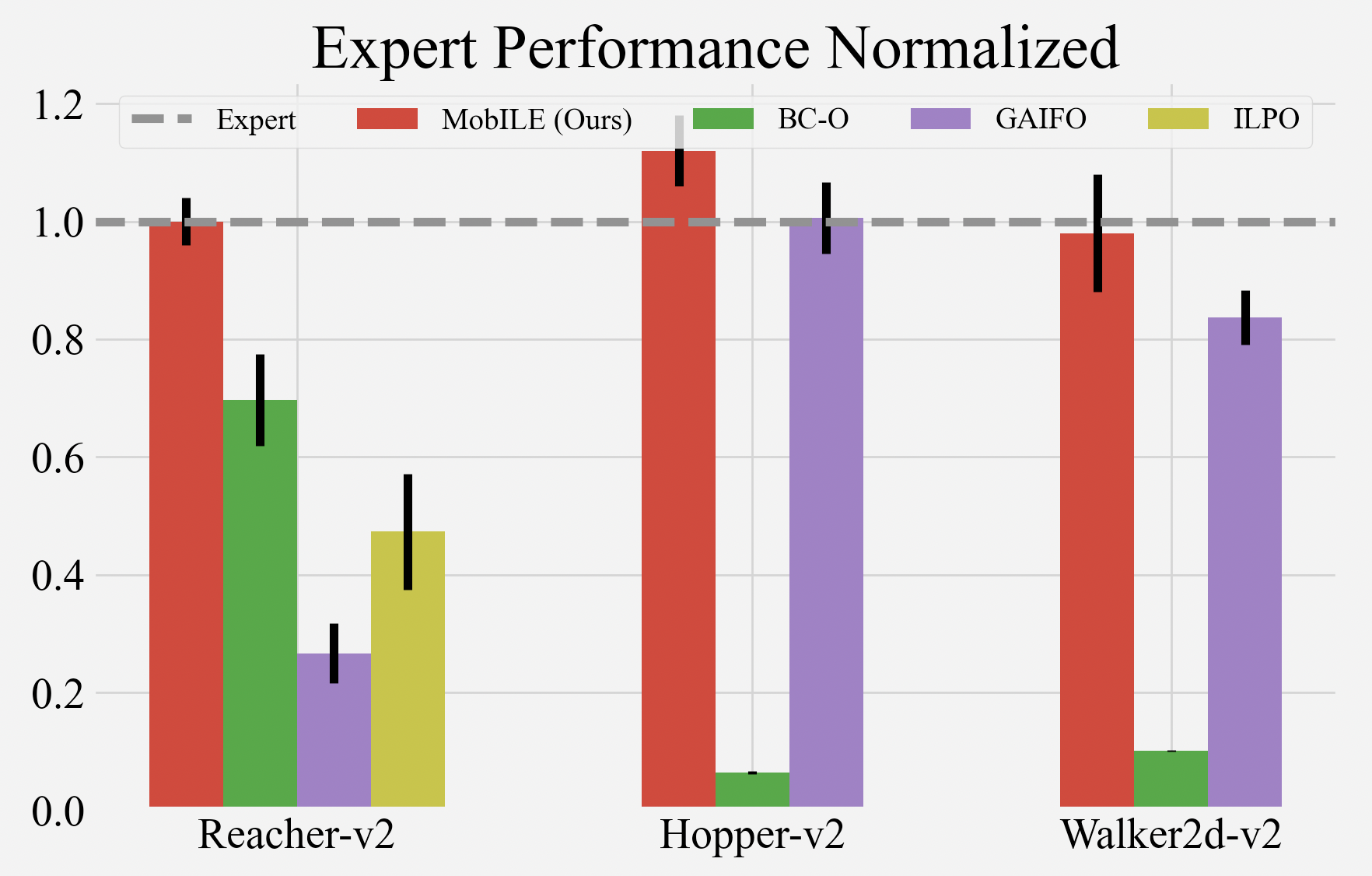}
  \end{center}
  \caption{Expert performance normalized scores of ILFO algorithms averaged across 5 seeds in environments with discrete action spaces (Reacher-v2) and continuous action spaces (Hopper-v2 and Walker2d-v2).}
  \label{fig:front_page}
\end{wrapfigure}
We introduce \algname---\underline{Mo}del-\underline{b}ased \underline{I}mitation \underline{L}earning and \underline{E}xploring, a model-based framework, to solve the ILFO problem. In contrast to existing model-based efforts, \algname~learns the forward transition dynamics model---a quantity that is well defined for any MDP. Importantly, \algname~\emph{combines strategic exploration with imitation} by interleaving a model learning step with a bonus-based, optimistic distribution matching step -- a perspective, to the best of our knowledge, that has not been considered in Imitation Learning. \algname~has the ability to automatically trade-off exploration and imitation. It simultaneously explores to collect data to refine the  model and imitates the expert wherever the learned model is accurate and certain. 
At a high level, our theoretical results and experimental studies demonstrate that
\emph{systematic exploration is beneficial for solving ILFO reliably and efficiently,}
and \emph{optimism} is a both theoretically sound and practically effective approach for strategic exploration in ILFO (see \pref{fig:front_page} for comparisons with other ILFO algorithms). This paper extends the realm of partial information problems (e.g. Reinforcement Learning and Bandits) where optimism has been shown to be crucial in obtaining strong performance, both in theory (e.g., $E^3$ \citep{kearns2002near}, UCB \citep{auer2002finite}) and practice (e.g., RND \citep{burda2018exploration}). This paper proves that incorporating optimism into the min-max IL framework~\citep{ziebart2008maximum,HoEr16GAIL,Sun19FAIL} is {\em beneficial} for both the theoretical foundations and empirical performance of ILFO.
\vspace{-3mm}
\paragraph{Our Contributions:} We present \algname~(Algorithm~\ref{alg:main_alg_bonus}), a provably efficient, model-based framework for ILFO that offers competitive results in benchmark gym tasks.
\algname~can be instantiated with various implementation choices owing to its modular design. This paper's contributions are:
\begin{enumerate}[leftmargin=0.5cm]
\item The~\algname~framework combines ideas of model-based learning, optimism for exploration, and adversarial imitation learning.~\algname~achieves global optimality with near-optimal regret bounds for classes of MDP dynamics that satisfy certain well studied notions of complexity. The key idea of~\algname~is to use optimism to \emph{trade-off imitation and exploration}. 
\item We show an exponential sample complexity gap between ILFO and classic IL where one has access to expert's actions.
This indicates that ILFO is \emph{fundamentally harder} than IL. Our lower bound on ILFO also indicates that to achieve near optimal regret, one needs to perform systematic exploration rather than random or no exploration, both of which will incur sub-optimal regret.  
\item We instantiate~\algname~with a model ensemble of neural networks and a disagreement-based bonus. We present experimental results on benchmark OpenAI Gym tasks, indicating \algname~compares favorably to or outperforms existing approaches. Ablation studies indicate that optimism indeed helps in significantly improving the performance in practice.
\end{enumerate}
\subsection{Related Works}
\noindent{\bf Imitation Learning} (IL) is considered through the lens of two types of approaches: (a) behavior cloning (BC)~\citep{pomerleau89} which casts IL as a reduction to supervised or full-information online learning~\citep{RossB10,DAgger}, or, (b) (adversarial) inverse RL~\citep{NgRussell00,PieterN04,ziebart2008maximum,FinnLA16Guided,HoEr16GAIL,ke2019imitation,ghasemipour2020divergence}, which involves minimizing various distribution divergences to solve the IL problem, either with the transition dynamics known (e.g., \cite{ziebart2008maximum}), or unknown (e.g., \cite{HoEr16GAIL}). \algname~does not assume knowledge of the transition dynamics, is model-based, and operates without access to the expert's actions.\\
\noindent{\bf Imitation Learning from Observation Alone} (ILFO) 
\citep{Sun19FAIL} presents a model-free approach \textsc{Fail} that outputs a non-stationary policy by reducing the ILFO problem into a sequence of min-max problems, one per time-step. While being theoretically sound, this approach cannot share data across different time steps and thus is not data efficient for long horizon problems. Also \textsc{Fail} in theory only works for discrete actions. In contrast, our paper learns a stationary policy using model-based approaches by reusing data across all time steps and extends to continuous action space.  
Another line of work~\citep{TorabiBCO,EdwardsILPO,YangInvDynDis} relies on learning an estimate of expert action, often through the use of an inverse dynamics models, $P^e(a|s,s')$. 
Unfortunately, an inverse dynamics model is not well defined in many benign problem instances. For instance, \cite[remark 1, section 9.3]{ZhuLDZ20} presents an example showing that inverse dynamics isn't well defined except in the case when the MDP dynamics is injective (i.e., no two actions could lead to the same next state from the current state. Note that even deterministic transition dynamics doesn't imply injectivity of the MDP dynamics).
Furthermore, ILPO~\citep{EdwardsILPO} applies to MDPs with deterministic transition dynamics and discrete actions. \algname, on the other hand, learns the forward dynamics model which is always unique and well-defined for both deterministic and stochastic transitions and works with discrete and continuous actions. Another line of work in ILFO revolves around using hand-crafted cost functions that may rely on task-specific knowledge~\citep{PengALP18,AytarPBP0F18,RLVideosSchmeckpeper}. The performance of policy outputted by these efforts relies on the quality of the engineered cost functions. 
In contrast,~\algname~does not require cost function engineering.\\ 
\noindent{\bf Model-Based RL} has seen several advances~\citep{Sutton90,LiT04,deisenroth2011pilco} including ones based on deep learning (e.g., \cite{LampeR14,GuLSL16,luo2018algorithmic,JannerFZL19,POLO,wang2019benchmarking}). Given \algname's modularity, these advances in model-based RL can be translated to improved algorithms for the ILFO problem. \algname~bears parallels to provably efficient model-based RL approaches including $E^3$~\citep{KearnsSingh2002, KakadeKL03}, R-MAX~\citep{Brafman2001RMAXA}, UCRL \citep{jaksch2010near}, UCBVI \citep{azar2017minimax}, Linear MDP \citep{yang2019reinforcement}, LC$^3$ \citep{kakade2020information}, Witness rank \citep{sun2019model} which utilize optimism based approaches to trade-off exploration and exploitation. Our work utilizes optimism to trade-off \emph{exploration and imitation}.
\section{Setting}\label{sec:setting}
We consider episodic finite-horizon MDP  $\Mcal = \{\Scal,\Acal, P^\star, H, c, s_0\}$, where $\Scal,\Acal$ are the state and action space, $P^\star:\Scal\times\Acal\mapsto \Scal$ is the MDP's transition kernel, H is the horizon, $s_0$ is a fixed initial state (note that our work generalizes when we have a distribution over initial states), and $c$ is the \emph{state-dependent} cost function $c: \Scal\mapsto [0,1]$. Our result can be extended to the setting where $c: \Scal\times\Scal\mapsto [0,1]$, i.e., the ground truth cost $c(s,s')$ depends on state and next state pairs. For analysis simplicity, we focus on $c:\Scal\mapsto [0,1]$.\footnote{Without any additional assumptions, in ILFO, learning to optimize action-dependent cost $c(s,a)$ (or $c(s,a,s')$ is {\bf not possible}. For example, if there are two sequences of actions that generate the same sequence of states, without seeing expert's preference over actions, we do not know which actions to commit to.}

We denote $d^{\pi}_{P}\in\Delta(\Scal\times\Acal)$ as the average state-action distribution of policy $\pi$ under the transition kernel $P$, i.e., $d^{\pi}_{P}(s,a):=\tfrac{1}{H}\sum_{t=1}^H Pr(s_t=s,a_t = a|s_0, \pi, P)$, where $Pr(s_t=s,a_t = a|s_0, \pi, P)$ is the probability of reaching $(s,a)$ at time step $t$ starting from $s_0$ by following $\pi$ under transition kernel $P$. We abuse notation and write $s\sim d^{\pi}_{P}$ to denote a state $s$ is sampled from the state-wise distribution which marginalizes action over $d^{\pi}_{P}(s,a)$, i.e., $d^{\pi}_{P}(s):=\tfrac{1}{H}\sum_{t=1}^H Pr(s_t=s|s_0,\pi,P)$. For a given cost function $f:\Scal\mapsto [0,1]$, $V^{\pi}_{P;f}$ denotes the expected total cost of $\pi$ under transition $P$ and cost function $f$. Similar to IL setting, in ILFO, the \emph{ground truth cost $c$ is unknown}. Instead, we can query the expert, denoted as $\pi^e: \Scal\mapsto \Delta(\Acal)$. Note that the expert $\pi^e$ could be stochastic and does not have to be the optimal policy. The expert, when queried, provides state-only demonstrations $\tau = \{s_0, s_1 \dots s_H\}$, where $s_{t+1} \sim P^\star(\cdot | s_t, a_t)$ and $a_t \sim \pi^e(\cdot | s_t)$.  


The goal is to leverage expert's state-wise demonstrations to learn a policy $\pi$ that performs as well as $\pi^e$ in terms of optimizing the ground truth cost $c$, with polynomial sample complexity on problem parameters such as horizon, number of expert samples and online samples and underlying MDP's complexity measures (see section~\ref{sec:analysis} for precise examples). We track the progress of any (randomized) algorithm by measuring the (expected) regret incurred by a policy $\pi$ defined as $E[V^\pi] - V^{\pi^*}$ as a function of number of online interactions utilized by the algorithm to compute $\pi$.
\subsection{Function Approximation Setup}
Since the ground truth cost $c$ is unknown, we utilize the notion of a function class (i.e., discriminators) $\Fcal \subset \Scal\mapsto [0,1]$ to define the costs that can then be utilized by a planning algorithm (e.g. NPG~\citep{Kakade01}) for purposes of distribution matching with expert states. If the ground truth $c$ depends $(s,s')$, we use discriminators $\Fcal\subset \Scal\times\Scal\mapsto [0,1]$. 
Furthermore, we use a model class $\Pcal \subset \Scal\times\Acal\mapsto \Delta(\Scal)$ to capture the ground truth transition $P^\star$. For the theoretical results in the paper, we assume realizability:
\begin{assum}\label{assum:realizable}
    Assume $\Fcal$ and $\Pcal$ captures ground truth cost and transition, i.e., $c\in\Fcal$, $P^\star\in\Pcal$. 
\end{assum}
We will use Integral probability metric (IPM) with $\Fcal$ as our divergence measure. Note that if $c\in\Fcal$ and $c:\Scal\mapsto [0,1]$, then IPM defined as $\max_{f\in\Fcal} \mathbb{E}_{s\sim d^{\pi}} f(s) - \mathbb{E}_{s\sim d^{\pi^e}}f(s)$ directly upper bounds sub-optimality gap $V^{\pi} - V^{\pi^e}$, where $V^{\pi}$ is the expected total cost of $\pi$ under cost function $c$. This justifies why minimizing IPM between two state distributions suffices~\citep{HoEr16GAIL, Sun19FAIL}. Similarly, if $c$ depends on $s,s'$, we can simply minimize IPM between two state-next state distributions, i.e., $\max_{f} \mathbb{E}_{s,s'\sim d^{\pi}} f(s,s') - \mathbb{E}_{s,s'\sim d^{\pi^e}}f(s,s')$ where discriminators now take $(s,s')$ as input.\footnote{we slightly abuse notation here and denote $d^{\pi}$ as the average state-next state distribution of $\pi$, i.e., $d^{\pi}(s,s') := d^{\pi}(s) \int_{a} \pi(a|s) da P^\star(s' |s ,a)$.}

To permit generalization, we require $\Pcal$ to have bounded complexity. For analytical simplicity, we assume  $\Fcal$ is discrete (but exponentially large), and we require the sample complexity of any PAC algorithm to scale polynomially with respect to its complexity $\ln(|\Fcal|)$. The $\ln| \Fcal |$ complexity can be replaced to bounded conventional complexity measures such as Rademacher complexity and covering number for continuous $\Fcal$ (e.g., $\Fcal$ being a Reproducing Kernel Hilbert Space). 

\section{Algorithm}\label{sec:alg}
\begin{algorithm}[t]
\caption{\algname: The framework of \textbf{Mo}del-\textbf{b}ased \textbf{I}mitation \textbf{L}earning and \textbf{E}xploring for ILFO }\label{alg:main_alg_bonus}
\begin{algorithmic}[1]
    \STATE \textbf{Require}: IPM class $\Fcal$, dynamics model class $\Pcal$, policy class $\Pi$, bonus function class $\Bcal$, expert dataset $\Dcal_e\equiv\{s^e_i\}_{i=1}^N$.
    \STATE Initialize policy $\pi_0\in\Pi$, replay buffer $\Dcal_{-1}=\emptyset$. 
    \FOR{$t = 0, \cdots ,{T-1}$}
        \STATE Execute $\pi_t$ in true environment $P^\star$ to get samples $\tau_t = \{s_k, a_k\}_{k=0}^{H-1} \cup s_H$. Append to replay buffer $\Dcal_t=\Dcal_{t-1}\cup\tau_t$.
        \STATE \textcolor{blue}{Update model and bonus}: $\widehat{P}_{t+1}:\Scal\times\Acal\to\Scal$ and $b_{t+1}:\Scal\times\Acal\to\mathbb{R}^+$ using buffer $\Dcal_t$. \label{line:model_fit}
        
	    \STATE \textcolor{blue}{Optimistic model-based min-max IL}: obtain $\pi_{t+1}$ by solving equation~(\ref{eq:model-based-ipm}) with $\widehat{P}_{t+1}, b_{t+1},\Dcal_e$.\label{line:planning}
    \ENDFOR
    \STATE  {\bf Return} $\pi_{T}$.
\end{algorithmic}
\end{algorithm}
\noindent We introduce~\algname~(Algorithm~\ref{alg:main_alg_bonus}) for the ILFO problem. 
~\algname~utilizes (a) a function class $\Fcal$ for Integral Probability Metric (IPM) based distribution matching, (b) a transition dynamics model class $\Pcal$ for model learning, (c) a bonus parameterization $\Bcal$ for exploration, (d) a policy class $\Pi$ for policy optimization.  
At every iteration, \algname~(in  \pref{alg:main_alg_bonus}) performs the following steps:
\begin{enumerate}[leftmargin=*]
    \itemsep 0em
    \item \textbf{Dynamics Model Learning:} execute policy in the environment online to obtain state-action-next state $(s,a,s')$ triples which are appended to the buffer $\Dcal$. Fit a transition model $\widehat{P}$ on $\Dcal$.
    \item \textbf{Bonus Design:} design bonus to incentivize exploration where the learnt dynamics model is uncertain, i.e. the bonus $b(s,a)$ is large at state $s$ where $\widehat{P}(\cdot | s,a)$ is uncertain in terms of estimating $P^\star(\cdot|s,a)$, while $b(s,a)$ is small where $\widehat{P}(\cdot | s,a)$ is certain.
    \item \textbf{Imitation-Exploration tradeoff:} Given discriminators $\Fcal$, model $\widehat{P}$, bonus $b$ and expert dataset $\Dcal_e$, perform distribution matching by solving the model-based IPM objective with bonus:
\begin{align}
        &\pi_{t+1}\leftarrow \arg\min_{\pi\in\Pi}\max_{f\in\Fcal}\  L(\pi,f;\widehat{P}, b, \Dcal_e):=   \mathbb{E}_{(s,a)\sim d^{\pi}_{\widehat{P}}} \left[f(s) - b(s,a)\right] - \mathbb{E}_{s\sim\Dcal_e} \left[f(s)\right],\label{eq:model-based-ipm}
\end{align} where $\EE_{s\sim \Dcal_e}f(s) : = \sum_{s\in\Dcal_e} f(s) / |\Dcal_e|$.
\end{enumerate}
Intuitively, the bonus cancels out discriminator's power in parts of the state space where the dynamics model $\widehat{P}$ is not accurate, thus offering freedom for \algname~to explore. We first explain \algname's components and then discuss \algname's key property---which is to trade-off \emph{exploration and imitation}. 

\subsection{Components of~\algname}
This section details \algname's components.

{\bf Dynamics model learning:} For the model fitting step in \pref{line:model_fit}, we assume that we get a calibrated model in the sense that:
$\| \widehat{P}_{t}(\cdot|s,a) - P^\star(\cdot|s,a) \|_{1} \leq \sigma_t(s,a), \forall s,a$ for some uncertainty measure $\sigma_t(s,a)$, similar to model-based RL works, e.g.~\citep{curi2020efficient}. We discuss ways to estimate $\sigma_t(s,a)$ in the bonus estimation below. There are many examples (discussed in \pref{sec:analysis}) that permit efficient estimation of these quantities including tabular MDPs, Kernelized nonlinear regulator, nonparametric model such as Gaussian Processes. 
Consider a general function class $\Gcal \subset \Scal\times\Acal\mapsto \Scal$, 
one can learn $\widehat{g}_t$ via solving a regression problem, i.e., 
\begin{align}\label{eq:model_learning}
\widehat{g}_t = \argmin_{g\in\Gcal} \sum_{s,a,s'\in\Dcal_t} \| g(s,a) - s' \|_2^2,
\end{align}
and setting $\widehat{P}_t(\cdot | s,a) = \Ncal\left( \widehat{g}_t(s,a), \sigma^2 I \right)$, where, $\sigma$ is the standard deviation of error induced by $\widehat{g}_t$. 
In practice, such parameterizations have been employed in several settings in RL with $\Gcal$ being a multi-layer perceptron (MLP) based function class (e.g.,\citep{RajeswaranGameMBRL}). In \pref{sec:analysis}, we also connect this with prior works in provable model-based RL literature.\\
{\bf Bonus:} We utilize bonuses as a means to incentivize the policy to efficiently explore unknown parts of the state space for improved model learning (and hence better distribution matching). With the uncertainty measure $\sigma_t(s,a)$ obtained from calibrated model fitting, we can simply set the bonus $b_t(s,a) = O(H \sigma_t(s,a))$. How do we obtain $\sigma_t(s,a)$ in practice? For a general class $\Gcal$, given the least square solution $\widehat{g}_t$, we can define a version space $\Gcal_t$ as:
    $\Gcal_t = \left\{g\in\Gcal: \sum_{i=0}^{t-1}\sum_{h=0}^{H-1} \| g(s_h^t,a_h^t) - \widehat{g}_t(s_h^t,a_h^t) \|_2^2 \leq z_t \right\}$,
with $z_t$ being a hyper parameter. The version space $\Gcal_t$ is an \emph{ensemble of functions} $g\in\Gcal$ which has training error on $\Dcal_t$ almost as small as the training error of the least square solution $\widehat{g}_t$. In other words, version space $\Gcal_t$ contains functions that agree on the training set $\Dcal_t$.  
The uncertainty measure at $(s,a)$ is then the \emph{maximum disagreement} among models in $\Gcal_t$, with $\sigma_t(s,a) \propto \sup_{f_1,f_2\in\Gcal_t} \| f_1(s,a) - f_2(s,a) \|_2$.  Since $g\in\Gcal_t$ agree on $\Dcal_t$, a large $\sigma_t(s,a)$ indicates $(s,a)$ is novel. See example~\ref{exp:general_G} for more theoretical details. 

Empirically, disagreement among an ensemble~\citep{OsbandAC18, Azizzadenesheli18, BurdaESK19, PathakG019, POLO} is used for designing bonuses that incentivize exploration.  We utilize an neural network ensemble
, where each model is trained on $\Dcal_t$ (via SGD on squared loss Eq.~\ref{eq:model_learning}) with different initialization. This approximates the version space $\Gcal_t$, and the bonus is set as a function of maximum disagreement among the ensemble's predictions.

{\bf Optimistic model-based min-max IL:} For model-based imitation (\pref{line:planning}), \algname~takes the current model $\widehat{P}_t$ and the discriminators ${\Fcal}$ as inputs and performs policy search to minimize the divergence defined by $\widehat{P}_n$ and ${\Fcal}$: $d_t(\pi, \pi^e) :=\max_{f\in{\Fcal}} \left[\EE_{s,a\sim d^{\pi}_{\widehat{P}_t}} (f(s) - b_t(s,a))  - \EE_{s\sim d^{\pi^e}} f(s) \right]$.
Note that, for a fixed $\pi$, the $\arg\max_{f\in\Fcal}$ is identical with or without the bonus term, since $\EE_{s,a\sim d^{\pi}_{\widehat{P}_t}} b_t(s,a)$ is independent of $f$. 
In our implementation, we use the Maximum Mean Discrepancy (MMD) with a Radial Basis Function (RBF) kernel to model discriminators $\Fcal$.\footnote{For MMD with kernel $k$, $\Fcal = \{ w^{\top} \phi(s,a) | \|w\|_2 \leq 1 \}$ where $\phi$: $\langle \phi(s,a), \phi(s',a') \rangle = k((s,a),(s',a'))$. } We compute $\argmin_{\pi} d_t(\pi, \pi^e)$ by iteratively (1) computing the $\argmax$ discriminator $f$ 
given the current $\pi$, and (2) using policy gradient methods (e.g., TRPO) to update $\pi$ inside $\widehat{P}_t$ with $f - b_t$ as the cost. Specifically, to find $\pi_t$ (\pref{line:planning}), we iterate between the following two steps:
\begin{align*} 
    \text{1. Cost update:}\hat{f} = \argmax_{f\in\Fcal} \EE_{s\sim d^{\hat\pi}_{\widehat{P}_t}} f(s) - \EE_{s\sim \Dcal^e} f(s),\quad \text{2. PG Step:}\hat{\pi} = \hat{\pi} - \eta \cdot \nabla_{\pi} V^{\hat\pi}_{\widehat{P}_t, \hat{f}-b_t},
\end{align*}
where the PG step uses the learnt dynamics model $\widehat{P}_t$ and the optimistic IPM cost $\hat{f}(s) - b_t(s,a)$. Note that for MMD, the cost update step has a closed-form solution. 
\subsection{Exploration And Imitation Tradeoff}
We note that \algname~is performing an automatic \emph{trade-off between exploration and imitation}. 
More specifically, the bonus is designed such that it has high values in the state space that have not been visited, and low values in the state space that have been frequently visited by the sequence of learned policies so far. 
Thus, by incorporating the bonus into the discriminator $f\in\Fcal$ (e.g., $\widetilde{f}(s,a) = f(s) - b_t(s,a)$), we diminish the power of discriminator $f$ at novel state-action space regions, which relaxes the state-matching constraint (as the bonus cancels the penalty from the discriminators) at those novel regions so that exploration is encouraged. For well explored states, we force the learner's states to match the expert's using the full power of the discriminators. Our work uses optimism (via coupling bonus and discriminators) to carefully balance imitation and exploration. 


\section{Analysis}
\label{sec:analysis}
This section presents a general  theorem for \algname~that uses the notion of \emph{information gain} \cite{srinivas2009gaussian}, and then specializes this result to common classes of stochastic MDPs such as discrete (tabular) MDPs, Kernelized nonlinear regulator \cite{KakadeKNR}, and general function class with bounded Eluder dimension \cite{RussoEluder}. 

Recall, \pref{alg:main_alg_bonus} generates one state-action trajectory $\tau^t := \{ s_h^t, a^t_h\}_{h=0}^{H}$ at iteration $t$ and estimates model $\widehat{P}_{t}$ based on $\Dcal_{t}= \tau^0,\dots, \tau^{t-1}$. We present our theorem under the assumption that model fitting gives us a model $\widehat{P}$ and a confidence interval of the model's prediction. \begin{assum}[Calibrated Model]
    \label{assum:model_calibrate}
    For all iteration $t$ with $t\in \mathbb{N}$, with probability $1-\delta$, we have a model $\widehat{P}_t$ and its associated uncertainty measure $\sigma_t: \Scal\times\Acal\mapsto \mathbb{R}^+$, such that for all $s,a\in \Scal\times\Acal$\footnote{the uncertainty measure $\sigma_t(s,a)$  will depend on the input failure probability $\delta$, which we drop here for notational simplicity. When we introduce specific examples, we will be explicit about the dependence on the failure probability $\delta$ which usually is in the order of $\ln(1/\delta)$.}
    \begin{align*}
        \left\| \widehat{P}_t(\cdot | s,a) - P^\star(\cdot | s,a) \right\|_{1} \leq \min\left\{\sigma_t(s,a), 2 \right\} . 
    \end{align*}
\end{assum}
Assumption~\ref{assum:model_calibrate} has featured in prior works (e.g., \cite{curi2020efficient}) to prove regret bounds in model-based RL. Below we demonstrate  examples that satisfy the above assumption. 

\begin{example}[Discrete MDPs] Given $\Dcal_t$, denote $N(s,a)$ as the number of times $(s,a)$ appears in $\Dcal_t$, and $N(s,a,s')$ number of times $(s,a,s')$ appears in $\Dcal_t$. We can set $\widehat{P}_t(s'|s,a) = N(s,a,s') / N(s,a), \forall s,a,s'$. We can set $\sigma_t(s,a) = \widetilde{O}\left(\sqrt{S / N(s,a)}\right)$.
\end{example}

\begin{example}[KNRs \cite{KakadeKNR}]
\label{exp:knr}
For KNR, we have $P^\star(\cdot | s,a) = \mathcal{N}\left( W^\star\phi(s,a), \sigma^2 I \right)$ where feature mapping $\phi(s,a) \in \mathbb{R}^{d}$ and $\|\phi(s,a)\|_2\leq 1$ for all $s,a$.\footnote{The covariance matrix can be generalized to any PSD matrix with bounded condition number.} 
We can learn $\widehat{P}_t$ via Kernel Ridge regression, i.e., $\widehat{g}_t(s,a) = \widehat{W}_t \phi(s,a)$ where $$\widehat{W}_t = \argmin_{W} \sum_{s,a,s'\in\Dcal_t} \left\|W\phi(s,a) - s' \right\|_2^2 + \lambda \left\|W \right\|_F^2$$ where $\|\cdot\|_F$ is the Frobenius norm.  
The uncertainty measure $\sigma_t(s,a) = \frac{\beta_t }{\sigma} \left\| \phi(s,a) \right\|_{\Sigma_t^{-1}}$, 
$\beta_t=  \{ 2\lambda \|W^\star\|^2_2  + 8 \sigma^2\cdot [d_s \ln(5) + 2 \ln( t^2 / \delta) + \ln(4) + \allowbreak\ln\left( \det(\Sigma_t) / \det(\lambda I)  \right)] \}^{1/2}
$,
and, $\Sigma_t = \sum_{k = 0}^{t-1} \sum_{h=1}^{H-1} \phi(s_h^k, a_h^k)\phi(s_h^k,a_h^k)^{\top} + \lambda I \text{ with } \lambda>0$.
See \pref{prop:knr_bonus} for more details.
\end{example}
Similar to RKHS, Gaussian processes (GPs) offers a calibrated model \citep{srinivas2009gaussian}.  Note that GPs offer similar regret bounds as RKHS; so we do not discuss GPs and instead refer readers to \cite{curi2020efficient}.
%


\begin{example}[General class $\Gcal$] \label{exp:general_G}
In this case, assume we have $P^\star(\cdot | s,a) = \mathcal{N}(g^\star(s,a), \sigma^2 I)$ with $g^\star\in\Gcal$. Assume $\Gcal$ is discrete (but could be exponentially large with complexity measure, $\ln(|\Gcal|)$), and $\sup_{g\in\Gcal,s,a} \|g(s,a)\|_2 \leq G\in\mathbb{R}^+$. Suppose model learning step is done by least square:
$\widehat{g}_t = \argmin_{g\in\Gcal} \sum_{k=0}^{t-1}\sum_{h=0}^{H-1} \left\| g(s_h^k,a_h^k) - s_{h+1}^k \right\|_2^2$.
Compute a version space $\Gcal_t = \left\{ g\in\Gcal: \sum_{k=0}^{t-1}\sum_{h=0}^{H-1} \left\| g(s_h^k,a_h^k) - \widehat{g}_t(s_h^k,a_h^k)  \right\|_2^2 \leq z_t \right\}$, where $z_t =  2\sigma^2 G^2 {\ln(2t^2 |\Gcal| / \delta) }$ and use this for uncertainty computation.  In particular, set uncertainty $\sigma_t(s,a) = \frac{1}{\sigma} \max_{g_1\in\Gcal ,g_2\in\Gcal} \| g_1(s,a) - g_2(s,a) \|_2$, i.e., the maximum disagreement between any two functions in the version space $\Gcal_t$. Refer to~\pref{prop:uncertainty_eluder} for more details. \looseness=-1
\end{example}
The maximum disagreement above motivates our practical implementation where we use an ensemble of neural networks to approximate the version space and use the maximum disagreement among the models' predictions as the bonus. We refer readers to \pref{sec:exp} for more details.

\subsection{Regret Bound}
We bound regret with the quantity named \emph{Information Gain} $\Ical$ (up to some constant scaling factor)~\citep{srinivas2009gaussian}:
\begin{align}
\label{eq:info_gain}
\Ical_T :=  \max_{\text{Alg}}\EE_{\text{Alg}} \left[ \sum_{t=0}^{T-1} \sum_{h=0}^{H-1} \min\left\{\sigma^2_t(s_h^t, a_h^t), 1\right\}\right],
\end{align}  where $\text{Alg}$ is any adaptive algorithm (thus including \pref{alg:main_alg_bonus}) that maps from history before iteration $t$ to some policy $\pi_t \in \Pi$. 
After the main theorem, we give concrete examples for $\Ical_T$ where we show that $\Ical_{T}$ has extremely mild growth rate with respect to $T$ (i.e., logarithimic).  Denote $V^{\pi}$ as the expected total cost of $\pi$ under the true cost function $c$ and the real dynamics $P^\star$. 

\begin{theorem}[Main result]  \label{thm:main_unified}Assume model learning is calibrated (i.e., \pref{assum:model_calibrate} holds for all $t$) and \pref{assum:realizable} holds.  In \pref{alg:main_alg_bonus}, set bonus $b_t(s,a): = H \min\{\sigma_t(s,a),2\}$. There exists a set of parameters, such that after running \pref{alg:main_alg_bonus} for $T$ iterations, we have:
\begin{align*}
    \EE\left[ \min_{t\in[0,\dots, T-1]} V^{\pi_t} - V^{\pi^e} \right] \leq 
    O\left(\frac{H^{2.5}\sqrt{\Ical_T}}{\sqrt{T}} + H \sqrt{\frac{\ln(T H |\Fcal|)}{ N }}\right).
\end{align*}
\end{theorem}
Appendix~\ref{sec:app_proofs} contains proof of Theorem~\ref{thm:main_unified}. This theorem indicates that as long as $\Ical_T$ grows sublinearly $o({T})$, we find a policy that is at least as good as the expert policy when $T$ and $N$ approach infinity. 
For any discrete MDP, KNR \cite{KakadeKNR}, Gaussian Processes models  \cite{srinivas2009gaussian},  and general $\Gcal$ with bounded Eluder dimension (\cite{russo2014learning,osband2014model}), we can show that the growth rate of $\Ical_T$ with respect to $T$ is mild. 

\begin{corollary}[Discrete MDP]\label{coro:discrete}
For discrete MDPs, $\Ical_{T} = \widetilde{O}(H{S^2A})$ where $S = |\Scal|, A = |\Acal|$.
Thus:
\begin{align*}
\EE\left[ \min_{t\in[0,\dots, T-1]} V^{\pi_t} - V^{\pi^e} \right] = \widetilde{O}\left( \frac{ H^3  S \sqrt{A}}{\sqrt{T}} + H \sqrt{ \frac{\ln( |\Fcal|)}{N}}    \right).
\end{align*}
\end{corollary}
Note that Corollary~\ref{coro:discrete} (proof in \pref{app:discrete_mdp}) hold for \emph{any} MDPs (not just injective MDPs) and any stochastic expert policy. The dependence on $A,T$ is tight (see lower bound in \pref{ssec:theory_explore}). Now we specialize \pref{thm:main_unified} to continuous MDPs below. 
\begin{corollary}[KNRs (Example~\ref{exp:knr})]For simplicity, consider the finite dimension setting $\phi:\Scal\times\Acal\mapsto \RR^d$. We can show that $\Ical_{T} = \widetilde{O}\left( H d + H d d_s + H d^2  \right)$ (see \pref{prop:IG_knr} for details), where $d$ is the dimension of the feature $\phi(s,a)$ and $d_s$ is the dimension of the state space. 
Thus, we have \footnote{We use $\widetilde{O}$ to suppress log term except the $\ln(|\Gcal|)$ and $\ln(|\Fcal|)$ which present the  complexity of $\Fcal$ and $\Gcal$.}
\begin{align*}
\EE\left[ \min_{t\in[0,\dots, T-1]} V^{\pi_t} - V^{\pi^e} \right] = \widetilde{O}\left( \frac{ H^3  \sqrt{d d_s + d^2}}{\sqrt{T}} + H \sqrt{ \frac{\ln( |\Fcal|)}{N}}    \right).
\end{align*}
\end{corollary}

\begin{corollary}[General $\Gcal$ with bounded Eluder dimension (Example~\ref{exp:general_G})]~\label{cor:generalG}~For general $\Gcal$, assume that $\Gcal$ has Eluder-dimension $d_{E}(\epsilon)$ (Definition 3 in \cite{osband2014model}). Denote $d_{E} = d_{E}(1/ TH)$. The information gain is upper bounded as $\Ical_T = {O}\left( H d_{E} + d_{E}\ln(T^3H  |\Gcal| )\ln(TH) \right)$ (see \pref{prop:eluder_bound_ig}). 
Thus, $$\EE\left[ \min_{t\in[0,\dots, T-1]} V^{\pi_t} - V^{\pi^e} \right] = \widetilde{O}\left( \frac{H^3 \sqrt{d_E \ln(TH|\Gcal| )}}{ \sqrt{T} } +  H \sqrt{\frac{ \ln( |\Fcal|)} { N}}  \right).$$
\end{corollary}

Thus as long as $\Gcal$ has bounded complexity in terms of the Eluder dimension \cite{russo2014learning,osband2014model}, \algname~with the maximum disagreement-based optimism leads to near-optimal guarantees. 

\subsection{Exploration in ILFO and the Exponential Gap between IL and ILFO}\label{ssec:theory_explore}
To show the benefit of strategic exploration over random exploration in ILFO, we present a {\em novel} reduction of the ILFO problem to a bandit optimization problem, for which strategic exploration is known to be {\em necessary}~\citep{BubeckC12} for optimal bounds while random exploration is suboptimal; this reduction indicates that benefit of strategic exploration for solving ILFO efficiently. This reduction also demonstrate that there exists an exponential gap in terms of sample complexity between ILFO and classic IL that has access to expert actions.  We leave the details of the reduction framework in \pref{app:low_bound}. The reduction allows us to derive the following lower bound for any ILFO algorithm.
\begin{theorem}
There exists an MDP with number of actions $A \geq 2$, such that even with infinitely many expert data, any ILFO algorithm must occur expected commutative regret $\Omega(\sqrt{AT})$.
    \label{thm:ILFO_lower_bound}
\end{theorem}
 Specifically we rely on the following reduction where solving ILFO, with even infinite expert data, is at least as hard as solving an MAB problem with the known optimal arm's mean reward which itself occurs the same worst case $\sqrt{AT}$ cumulative regret bound as the one in the classic MAB setting. For MAB, it is known that random exploration such as $\epsilon$-greedy will occur suboptimal regret $O(T^{2/3})$. Thus to achieve optimal $\sqrt{T}$ rate, one needs to leverage strategic exploration (e.g., optimism).

Methods such as BC for IL have sample complexity that scales as $\text{poly}\ln(A)$, e.g., see \cite[Theorem 14.3, Chapter 14]{agarwal2019reinforcement} which shows that for tabular MDP, BC learns a policy whose performance is $O(H^2 \sqrt{ S\ln(A)  / N})$ away from the expert's performance (here $S$ is the number of states in the tabular MDP). Similarly, in interactive IL setting, DAgger~\cite{DAgger} can also achieve poly $\ln(A)$ dependence in sample complexity.  The \emph{exponential gap} in the sample complexity dependence on $A$ between IL and ILFO formalizes the  additional difficulty encountered by learning algorithms in ILFO. 
\section{Practical Instantiation of~\algname}\label{sec:prac_inst}
We present a brief practical instantiation~\algname's components with details in Appendix~\pref{sec:implementation_details}.\\
{\bf Dynamics model learning:}We employ Gaussian Dynamics Models parameterized by an MLP~\citep{RajeswaranGameMBRL,MOReL}, i.e., $\widehat{P}(s,a):=\mathcal{N}(h_{\theta}(s,a), \sigma^2 I)$, where, $h_\theta(s,a) = s + \sigma_{\Delta_s}\cdot \text{MLP}_\theta(s_c,a_c)$, where, $\theta$ are MLP's trainable parameters, $s_c = (s-\mu_s)/\sigma_s$, $a_c = (a-\mu_a)/\sigma_a$ with $\mu_s,\mu_a$ (and $\sigma_s,\sigma_a$) being the mean of states, actions (and standard deviation of states and actions) in the replay buffer $\mathcal{D}$. Next, for $(s,a,s')\in\mathcal{D}$, $\Delta_s=s'-s$ and $\sigma_{\Delta_s}$ is the standard deviation of the state differences $\Delta_s\in\Dcal$. We use SGD with momentum~\citep{SutskeverMomentum} for training the parameters $\theta$ of the MLP.\\
{\bf Discriminator parameterization:}We utilize MMD as our choice of IPM and define the discriminator as $f(s) = w^\top \psi(s)$, where, $\psi(s)$ are Random Fourier Features~\citep{rahimi2008random}.\\
{\bf Bonus parameterization:}We utilize the discrepancy between predictions of a pair of dynamics models $h_{\theta_1}(s,a)$ and $h_{\theta_2}(s,a)$ for designing the bonus. Empirically, we found that using more than two models in the ensemble offered little to no improvements.
Denote the disagreement at any $(s,a)$ as $\delta(s,a) =  \ \left\| h_{\theta_1}(s,a)-h_{\theta_2}(s,a) \right\|_2$, and $\delta_{\mathcal{D}} = \max_{(s,a)\sim\mathcal{D}} \delta(s,a)$ is the max discrepancy of a replay buffer $\Dcal$. We set bonus as $b(s,a) =\lambda \cdot \min(\delta(s,a)/\delta_\mathcal{D}$, where $\lambda>0$ is a tunable parameter.\\
{\bf PG oracle:}We use TRPO~\citep{SchulmanTRPO} to perform incremental policy optimization inside the learned model. 
\section{Experiments}
\begin{figure*}[t]
    \centering
    \begin{subfigure}
        \centering
        \includegraphics[width=\textwidth]{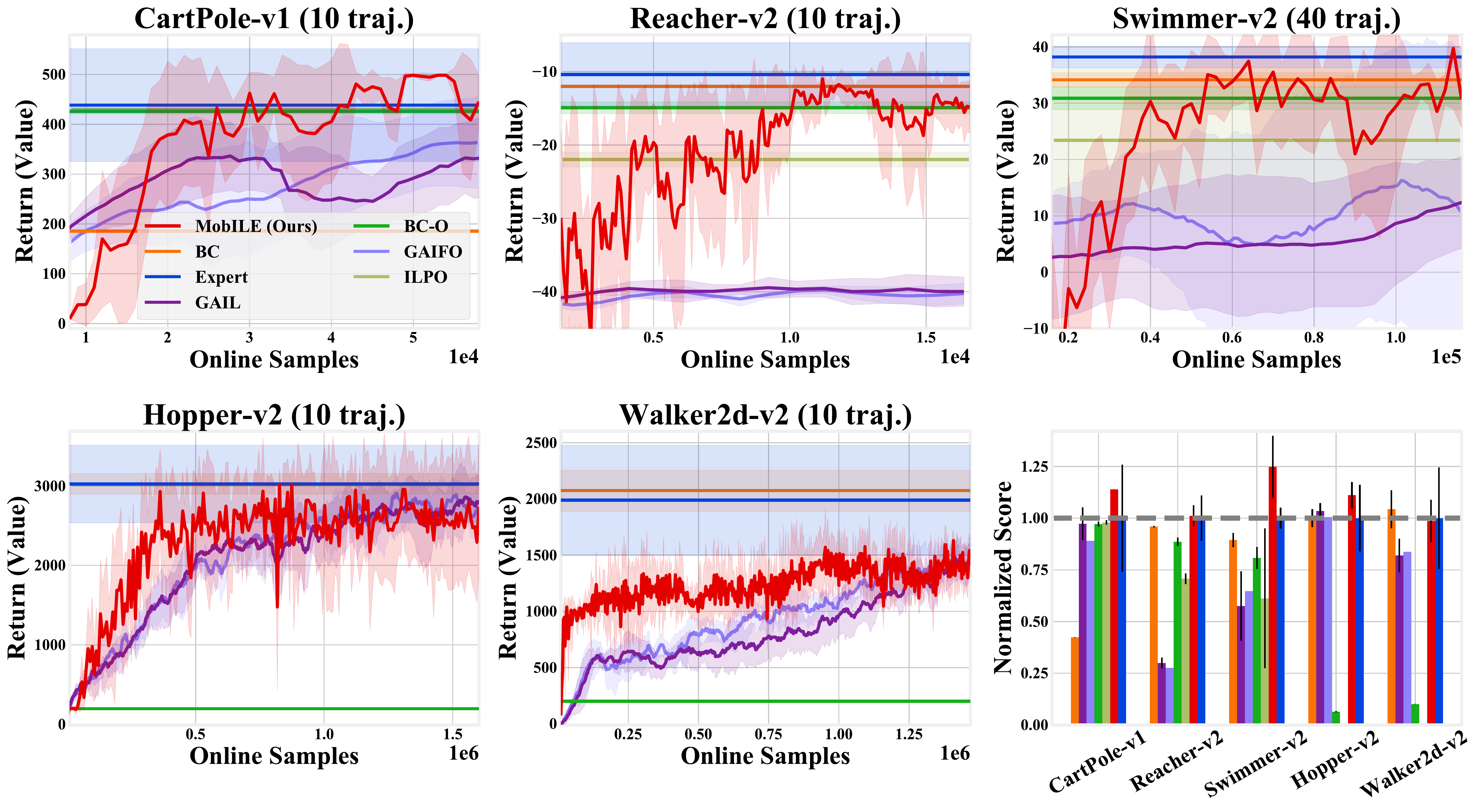}
    \end{subfigure}
    \vspace{-2mm}
    \caption{
    Comparing \algname~(red) against BC (orange), BC-O (green), GAIL (purple), GAIFO (periwinkle), ILPO (green olive). The learning curves are obtained by averaging all algorithms over $5$ seeds. \algname~outperforms BC-O, GAIL and matches BC's behavior despite \algname~not having access to expert actions. The bar plot (bottom-right) presents the best performing policy outputted by each algorithm averaged across $5$ seeds for each algorithm. \algname~clearly outperforms BC-O, GAIFO, ILPO while matching the behavior of IL algorithms like BC/GAIL which use expert actions.}
    \label{fig:performance}
    \vspace{-2mm}
\end{figure*}
\label{sec:exp}
This section seeks to answer the following questions: (1) How does \algname{} compare against other benchmark algorithms?
    (2) How does optimism impact sample efficiency/final performance?
    (3) How does increasing the number of expert samples impact the quality of policy outputted by \algname?

We consider tasks from Open AI Gym~\citep{brockman2016openai} simulated with Mujoco~\citep{todorov2012mujoco}: \cartpole, \reacher, \swimmer, \hopper~and \walker.
We train an expert for each task using TRPO \citep{SchulmanTRPO} until we obtain an expert policy of average value $460, -10, 38, 3000, 2000$ respectively. We setup \swimmer, \hopper,\walker{} similar to prior model-based RL works~\citep{KurutachCDTA18,nagabandi2018neural,luo2018algorithmic,RajeswaranGameMBRL,MOReL}. 

We compare \algname~against the following algorithms: Behavior Cloning (BC), GAIL~\citep{HoEr16GAIL},  BC-O~\citep{TorabiBCO}, ILPO~\citep{EdwardsILPO} (for environments with discrete actions), GAIFO~\citep{torabi2018gaifo}. Furthermore, recall that BC and GAIL utilize both expert states and actions, information that is not available for ILFO. This makes both BC and GAIL idealistic targets for comparing ILFO methods like \algname~against. As reported by Torabi et al.~\citep{TorabiBCO}, BC outperforms BC-O in all benchmark results. Moreover, our results indicate \algname~outperforms GAIL and GAIFO in terms of sample efficiency. With reasonable amount of parameter tuning, BC serves as a very strong baseline and nearly solves \emph{deterministic} Mujoco environments. We use code released by the authors for BC-O and ILPO. For GAIL we use an open source implementation~\citep{stable-baselines}, and for GAIFO, we modify the GAIL implementation as described by the authors. We present our results through (a) learning curves obtained by averaging the progress of the algorithm across $5$ seeds, and, (b) bar plot showing expert normalized scores averaged across $5$ seeds using the best performing policy obtained with each seed. Normalized score refers to ratio of policy's score over the expert score (so that expert has normalized score of 1). For \reacher, since the expert policy has a negative score, we add an constant before normalization. More details can be found in Appendix~\ref{sec:implementation_details}.
\begin{figure*}[b]
    \centering
    \begin{subfigure}
        \centering
        \includegraphics[width=\textwidth]{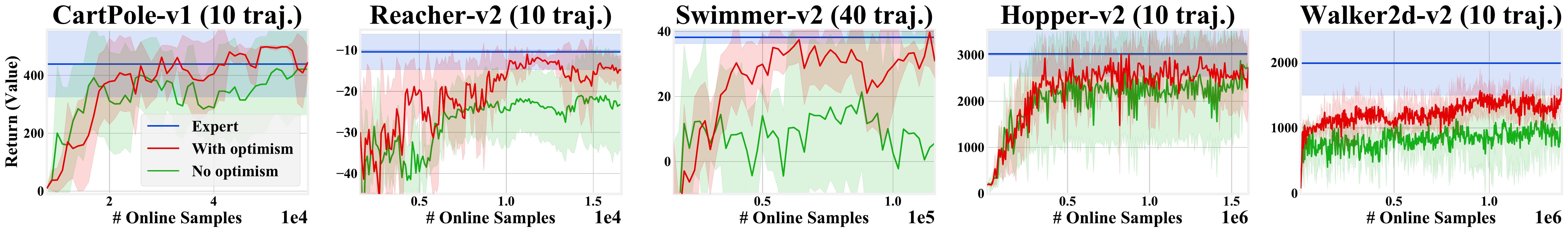}
    \end{subfigure}
    \vspace{-2mm}
    \caption{Learning curves obtained by running \algname~with (red) and without (green) optimism. Without optimism, the algorithm learns slowly or does not match the expert, whereas, with optimism, \algname~shows improved behavior by automatically trading off exploration and imitation.}\label{fig:bonus}
    \vspace{-4mm}
\end{figure*}
\subsection{Benchmarking~\algname~on MuJoCo suite}\label{ssec:expt_benchmark} Figure~\ref{fig:performance} compares \algname~with BC, BC-O, GAIL, GAIFO and ILPO. \algname~consistently matches or exceeds BC/GAIL's performance \emph{despite BC/GAIL having access to actions taken by the expert} and \algname~functioning {\em without} expert action information. \algname, also, consistently improves upon the behavior of ILFO methods such as BC-O, ILPO, and GAIFO. We see that BC does remarkably well in these benchmarks owing to determinism in the transition dynamics; in the appendix, we consider a variant of the cartpole environment with stochastic dynamics. Our results suggest that BC struggles with stochasticity in the dynamics and fails to solve this task, while \algname~continues to reliably solve this task. Also, note that we utilize $10$ expert trajectories for all environments except \swimmer; this is because all algorithms (including \algname) present results with high variance. We include a learning curve for \swimmer~with $10$ expert trajectories in the appendix. The bar plot in Figure~\ref{fig:performance} shows that within the sample budget shown in the learning curves, \algname\ (being a model-based algorithm), presents superior performance in terms of matching expert, thus indicating it is more sample efficient than GAIFO, GAIL (both being model-free methods), ILPO and BC-O.

\subsection{Importance of the optimistic MDP construction}\label{ssec:expt_ablation}
Figure~\ref{fig:bonus} presents results obtained by running \algname~with and without optimism. In the absence of optimism, the algorithm either tends to be sample inefficient in achieving expert performance or completely fails to solve the problem. Note that without optimism, the algorithm isn't explicitly incentivized to explore -- only implicitly exploring due to noise induced by sampling actions. This, however, is not sufficient to solve the problem efficiently. In contrast, \algname~with optimism presents improved behavior and in most cases, solves the environments with fewer online interactions. 
\subsection{Varying Number of Expert Samples}\label{ssec:expert_ablation}
\begin{wraptable}[7]{r}{0.5\textwidth}
    \vspace{-12mm}
    \centering
    \caption{Expert normalized score and standard deviation of policy outputted by \algname~when varying number of expert trajectories as $E_1$ and $E_2$ (specific values represented in parentheses)}
    \resizebox{0.5\textwidth}{!}{
    \begin{tabular}{c|ccc}
    \toprule
        Environment &  $E_1$ & $E_2$ & Expert \\
    \midrule
        \cartpole      & $1.07\pm0.15 \ (5)$ & $1.14\pm0\ (10)$ & $1\pm0.25$\\
        \reacher      & $1.01\pm0.05\ (10)$ & $0.997\pm0.055\ (20)$& $1\pm0.11$ \\
        \swimmer & $1.54\pm1.1\ (10)$& $1.25\pm0.15\ (40)$& $1\pm0.05$ \\
        \hopper         &$1.11\pm0.064\ (10)$ &$1.16\pm0.03\ (40)$ &$1\pm0.16$ \\
        \walker    & $0.975\pm0.12\ (10)$ & $0.94\pm0.038\ (50)$ & $1\pm0.25$ \\
    \bottomrule
    \end{tabular}
    }
    \label{tab:exp_ablate}
\end{wraptable}
Table~\ref{tab:exp_ablate} shows the impact of increasing the number of samples drawn from the expert policy for solving the ILFO problem. The main takeaway is that increasing the number of expert samples aids \algname~in reliably solving the problem (i.e. with lesser variance).

\section{Conclusions}\label{sec:discuss}
This paper introduces~\algname, a model-based ILFO approach that is applicable to MDPs with stochastic dynamics and continuous action spaces. \algname~trades-off exploration and imitation, and this perspective is shown to be important for solving the ILFO efficiently both in theory and in practice. Future works include exploring other means for learning dynamics models, performing strategic exploration and extending \algname~to problems with rich observation spaces (e.g. videos).

By not even needing the actions to imitate, ILFO algorithms allow for learning algorithms to capitalize on large amounts of video data available online. Moreover, in ILFO, the learner is successful if it learns to imitate the expert. Any expert policy designed by bad actors can naturally lead to obtaining new policies that continue to imitate and be a negative influence to the society. With this perspective in mind, any expert policy must be thoroughly vetted in order to ensure ILFO algorithms including \algname~are employed in ways that benefit the society.


\section*{Acknowledgements}
Rahul Kidambi acknowledges funding from NSF TRIPODS Award $\text{CCF}-1740822$ at Cornell University. All content represents the opinion of the authors, which is not necessarily shared or endorsed by their respective employers and/or sponsors.

\newpage
\bibliographystyle{abbrv}

\newpage
\tableofcontents
\appendix
\section{Analysis of \pref{alg:main_alg_bonus}}
\label{sec:app_proofs}
We start by presenting the proof for the unified main result in \pref{thm:main_unified}. We then discuss the bounds for special instances individually. 

The following lemma shows that under \pref{assum:model_calibrate}, with $b_t(s,a) = H \min\{\sigma_t(s,a),2\}$, we achieve \emph{optimism} at all iterations. 

\begin{lemma}[Optimism] Assume \pref{assum:model_calibrate} holds, and set $b_t(s,a) = H \min\left\{\sigma_t(s,a),2\right\}$.  For all state-wise cost function $f: \Scal\mapsto [0,1]$, denote the bonus enhance cost as $\widetilde{f}_t(s,a) := f(s) - b_t(s,a)$. For all policy $\pi$, we have the following optimism:
 \begin{align*}
 V^{\pi}_{\widehat{P}_t, \widetilde{f}_t} \leq V^{\pi}_{P, f}, \forall t. 
 \end{align*} 
 \label{lem:optimism}
\end{lemma}
\begin{proof} 
In the proof, we drop subscript $t$ for notation simplicity. We consider a fixed function $f$ and policy $\pi$. 
Also let us denote $\widehat{V}^{\pi}$ as the value function of $\pi$ under $(\widehat{P}, \widetilde{f})$, and $V^{\pi}$ as the value function under $(P, f)$.

Let us start from $h = H$, where we have $\widehat{V}^{\pi}_{H}(s) = V^{\pi}_H(s) = 0$. Assume inductive hypothesis holds at $h+1$, i.e., for any $s,a$, we have $\widehat{Q}^{\pi}_{h+1}(s,a) \leq Q^{\pi}_{h+1}(s,a)$. Now let us move to $h$. We have:
\begin{align*}
\widehat{Q}^{\pi}_h(s,a) - Q^{\pi}_h(s,a) & =  \widetilde{f}(s,a) + \EE_{s'\sim \widehat{P}(\cdot | s,a)} \widehat{V}^{\pi}_{h+1}(s') - {f}(s) - \EE_{s'\sim {P}(\cdot | s,a)} {V}^{\pi}_{h+1}(s') \\
& \leq -H \min\{\sigma(s,a),2\}  +  \EE_{s'\sim \widehat{P}(\cdot | s,a)} {V}^{\pi}_{h+1}(s')  - \EE_{s'\sim {P}(\cdot | s,a)} {V}^{\pi}_{h+1}(s') \\
& \leq - H \min\{\sigma(s,a),2\}  + H \left\| \widehat{P}(\cdot | s,a) - P(\cdot | s,a) \right\|_1 \\
& \leq - H \min\{ \sigma(s,a),2\}  + H \min\{\sigma(s,a),2\} = 0,
\end{align*}  where the first inequality uses the inductive hypothesis at time step $h+1$. 
Finally, note that $V^{\pi}_h(s) = \EE_{a\sim \pi(s)} Q^{\pi}_h(s,a)$, which leads to $\widehat{V}^{\pi}_h(s) \leq V^{\pi}_h(s)$. 
This concludes the induction step. 
\end{proof}
The next lemma concerns the statistical error from finite sample estimation of $\EE_{s \sim d^{\pi^e}} f(s)$.  
\begin{lemma}Fix $\delta \in (0,1)$. For all $t$, we have that with probability at least $1-\delta$, 
\begin{align*}
\left\lvert \EE_{s\sim d^{\pi^e}} f(s) -  \sum_{i=1}^N f(s^e_i) / N   \right\rvert \leq 2\sqrt{ \frac{ \ln\left( 2 t^2  |\Fcal | / \delta \right)  }{ N }}, \forall f \in{\Fcal}. 
\end{align*}
\label{lem:concentration}
\end{lemma}
\begin{proof}
For any $t$, we set the failure probability to be $6 \delta /  (t^2 \pi^2) $ at iteration $t$ where we abuse notation and point out that $\pi = 3.14159...$. Thus the total failure probability for all $t\in\mathbb{N}$ is at most $\delta$. We then apply classic Hoeffding inequality to bound $\EE_{s\sim d^{\pi^e}} f(s) -  \sum_{i=1}^N f(s^e_i) / N$ with the fact that $f(s) \in [0,1]$ for all $s$. We conclude the proof by taking a union bound over all $f\in\Fcal$. 
\end{proof}

Note that here we have assumed $s_i^e \sim d^{\pi^e}$ is i.i.d sampled from $d^{\pi^e}$. This can easily be achieved by randomly sampling a state from each expert trajectory.  Note that we can easily deal with i.i.d trajectories, i.e., if our expert data contains $N$ many i.i.d trajectories $\{\tau^1,\dots, \tau^{N}\}$, we can apply concentration on the trajectory level, and get:
\begin{align*}
\left\lvert \mathbb{E}_{\tau\sim \pi^e} \left[ \sum_{h=0}^{H-1} f(s_h) \right]   - \frac{1}{N} \sum_{i=1}^{N} \sum_{h=0}^{H-1} f(s^i_{h}) \right\rvert \leq O\left( H \sqrt{ \frac{\ln(t^2 |\Fcal| / \delta)}{N}  }\right),
\end{align*} where $\tau\sim \pi$ denotes that a trajectory $\tau$ being sampled based on $\pi$, $s_h^i$ denotes the state at time step $h$ on the i-th expert trajectory. Also note that we have $\mathbb{E}_{s\sim d^{\pi}} f(s) = \frac{1}{H} \mathbb{E}_{\tau \sim \pi} \left[ \sum_{h=0}^{H-1} f(s_{h})  \right] $ for any $\pi,f$. Together this immediately implies that:
\begin{align*}
\left\lvert \mathbb{E}_{s\sim d^{\pi^e}} f(s) -  \frac{1}{NH} \sum_{i=1}^{N} \sum_{h=0}^{H-1} f(s^i_{h}) \right\rvert \leq O\left(\sqrt{ \frac{\ln(t^2 |\Fcal| / \delta)}{N}  }\right),
\end{align*} which matches to the bound in \pref{lem:concentration}.

Now we conclude the proof for \pref{thm:main_unified}.
\begin{proof}[Proof of \pref{thm:main_unified}]
Assume that \pref{assum:model_calibrate} and the event in \pref{lem:concentration} hold. Denote the joint of these two events as $\Ecal$. Note that the probability of $\overline{\Ecal}$ is at most $2\delta$.  For notation simplicity, denote $\epsilon_{stats} = 2\sqrt{ \frac{ \ln\left( 2 T^2  |\Fcal | / \delta \right)  }{ N }}$. 

In each model-based planning phase, recall that we perform model-based optimization on the following objective:
\begin{align*}
\pi_{t} = \argmin_{\pi\in \Pi} \max_{{f} \in {F}} \left[ \mathbb{E}_{s,a \sim d^{\pi}_{\widehat{P}_t} } \left[f(s) - b_t(s,a)\right]  - \sum_{i=1}^N f({s}^e_i) / N  \right].
\end{align*}
Note that for any $\pi$, using the inequality in \pref{lem:concentration}, we have:
\begin{align*}
&\max_{{f} \in {\Fcal}_t} \left[ \mathbb{E}_{s,a \sim d^{\pi}_{\widehat{P}_t} } ({f}(s) - b_t(s,a))  - \sum_{i=1}^N {f}({s}^e_i) / N  \right]  \\
& = \max_{{f} \in {\Fcal}} \left[ \mathbb{E}_{s,a \sim d^{\pi}_{\widehat{P}_t} } ({f}(s) - b_t(s,a))  -   \EE_{s\sim d^{\pi^e}} f(s) +   \EE_{s\sim d^{\pi^e}} {f}(s)-  \sum_{i=1}^N {f}({s}^e_i) / N  \right] \\
& \leq  \max_{{f} \in {\Fcal}}\left[ \mathbb{E}_{s,a \sim d^{\pi}_{\widehat{P}_t} } ({f}(s) - b_t(s,a))  -   \EE_{s\sim d^{\pi^e}} {f}(s)   \right] + \max_{{f}\in {F} } \left[  \EE_{s\sim d^{\pi^e}} {f}(s)-  \sum_{i=1}^N {f}({s}^e_i) / N  \right] \\
& \leq  \max_{{f} \in {\Fcal}}\left[ \mathbb{E}_{s,a \sim d^{\pi}_{\widehat{P}_t} } \left({f}(s)-b_t(s,a)\right)  -   \EE_{s,a\sim d^{\pi^e}_{\widehat{P}_t}} \left({f}(s) - b_t(s,a)\right)   \right] + \epsilon_{stats}
\end{align*} where in the last inequality we use optimism from \pref{lem:optimism}, i.e., $\EE_{s,a\sim d^{\pi^e}_{\widehat{P}_t}} ({f}(s)  - b_t(s,a)) \leq \EE_{s\sim d^{\pi^e}} f(s)$.

Hence, for $\pi_{t}$, since it is the minimizer and $\pi^e\in\Pi$, we must have:
\begin{align*}
    &\max_{{f} \in {\Fcal}} \left[ \mathbb{E}_{s,a \sim d^{\pi_{t}}_{\widehat{P}_t} } \left({f}(s) - b_t(s,a)\right)  - \sum_{i=1}^N f({s}^e_i) / N  \right] \\
    &\leq \max_{{f} \in {\Fcal}} \left[ \mathbb{E}_{s,a \sim d^{\pi^e}_{\widehat{P}_t} } ({f}(s)-b_t(s,a))  - \sum_{i=1}^N {f}({s}^e_i) / N  \right] \\
    &\leq \max_{{f} \in {\Fcal}}\left[ \mathbb{E}_{s,a \sim d^{\pi^e}_{\widehat{P}_t} } ({f}(s)-b_t(s,a))  -   \EE_{s,a\sim d^{\pi^e}_{\widehat{P}_t}} ({f}(s)-b_t(s,a))   \right] + \epsilon_{stats} = \epsilon_{stats}.
\end{align*}
Note that ${\Fcal}$ contains ${c}$, we must have:
\begin{align*}
    \EE_{s,a\sim d^{\pi_t}_{\widehat{P}_t}} \left[{c}(s) - b_t(s,a)\right]   \leq \sum_{i=1}^N c(s_i^e) / N + \epsilon_{stats} \leq \EE_{s\sim d^{\pi^e}} c(s) + 2\epsilon_{stats},
\end{align*}  which means that $V^{\pi_t}_{\widehat{P}_t; \widetilde{c}_t} \leq V^{\pi^e} + 2H \epsilon_{stats}$. 

Now we compute the  regret in episode $t$. First recall that $b_t(s,a) = H\min\{ \sigma_t(s,a) , 2 \}$, which means that $\| b_t\|_{\infty} \leq 2H$ as $\|c\|_{\infty} \leq 1$, which means that $\left\| c - b_t \right\|_{\infty} \leq 2H$. Thus, $\left\| V^{\pi}_{\widehat{P};{c}-b_t} \right\|_{\infty} \leq 2 H^2$.
Recall simulation lemma (\pref{lem:simulation}), we have:
\begin{align*}
    V^{\pi_{t}} - V^{\pi^e} & \leq V^{\pi_{t}} - V^{\pi_{t}}_{\widehat{P}_t;\widetilde{c}_t} + 2 H \epsilon_{stats} \\
    & =  H \EE_{s,a\sim d^{\pi_{t}}} \left[ \left\lvert \widetilde{c}_t(s,a) - c(s) \right\rvert   + 2H^2 \left\| \widehat{P}_t(\cdot | s,a) - P^\star(\cdot | s,a)  \right\|_1   \right]  +  2 H \epsilon_{stat} \\
    & = H\EE_{s,a\sim d^{\pi_{t}}} \left[ H  \min\{\sigma_t(s,a),2\}   + 2H^2  \left\| \widehat{P}_t(\cdot | s,a) - P^\star(\cdot | s,a)  \right\|_1   \right]  +  2H \epsilon_{stat} \\ 
    & \leq H \EE_{s,a\sim d^{\pi_{t}}} \left[ H \min\{ \sigma_t(s,a),2\}   + 2H^2  \min\{\sigma_t(s,a),2\}  \right]  +  2H \epsilon_{stat}\\
    & \leq 3  H^3 \EE_{s,a \sim d^{\pi_{t}}} \min\{\sigma_t(s,a),2\} + 2H \epsilon_{stat} \\
    & \leq 6  H^3  \EE_{s,a\sim d^{\pi_{t}}} \min\{\sigma_t(s,a),1\} + 2H \epsilon_{stat}
\end{align*}

Now sum over $t$, and denote $\EE_{\pi_t}$ as the conditional expectation conditioned on the history from iteration $0$ to $t-1$,  we get:
\begin{align*}
\sum_{t=0}^{T-1}  \left[V^{\pi_{t}} - V^{\pi^e}\right] &  \leq 6H^2 \sum_{t=0}^{T-1} \EE_{\pi_t}\left[ \sum_{h=0}^{H-1}  \min\{ \sigma_t(s_h^t,a_h^t), 1\}\right] + 2H T \epsilon_{stat} \\
& \leq 6H^2 \sum_{t=0}^{T-1}  \left[ \sqrt{H} \sqrt{ \EE_{\pi_t}\sum_{h=0}^{H-1} \min\{\sigma_t^2(s_h^t,a_h^t), 1\}}  \right] + 2HT\epsilon_{stat}, 
\end{align*} where in the last inequality we use $\EE[a^{\top} b] \leq \sqrt{\EE[\|a\|^2_2]\EE[\|b\|^2_2]}$.

Recall that $\pi_t$ are random quantities, add expectation on both sides of the above inequality, and consider the case where $\Ecal$ holds and $\overline{\Ecal}$ holds,  we have:
\begin{align*}
\EE\left[\sum_{t=0}^{T-1}  \left(V^{\pi_{t}} - V^{\pi^e}\right) \right] & \leq 6H^{2.5}  \EE\left[ \sum_{t=0}^{T-1} \sqrt{ \EE_{\pi_t}\sum_{h=0}^{H-1} \min\left\{ \sigma_t^2(s_h^t,a_h^t), 1 \right\}} \right] + 2HT \epsilon_{stat} +  \PP(\overline{\Ecal}) TH \\
& \leq 6H^{2.5} \left[  \sqrt{T } \sqrt{  \EE\left[ \sum_{t=0}^{T-1} \sum_{h=0}^{H-1} \min\left\{\sigma^2_t(s_h^t,a_h^t), 1\right\}  \right]}     \right] + 2HT \epsilon_{stat} +  2\delta TH,
\end{align*}  where in the last inequality, we use $\EE[a^{\top} b] \leq \sqrt{ \EE[ \|a\|_2^2] \EE[\|b\|_2^2] }$.
This implies that that:
\begin{align*}
\EE\left[\min_{t}V^{\pi_t} - V^{\pi^e} \right]  \leq \frac{6H^{2.5}}{ \sqrt{T} }   \sqrt{   \max_{\text{Alg}}\EE_{\text{Alg}}\left[ \sum_{t=0}^{T-1} \sum_{h=0}^{H-1} \min\left\{ \sigma_t^2(s_h^t,a_h^t), 1 \right\}\right] }  + 2 H \epsilon_{stats} + 2 H \delta.
\end{align*} Set $\delta = 1/(HT)$, we get:
\begin{align*}
\EE\left[V^{\pi} - V^{\pi^e} \right] &  \leq \frac{6H^{2.5}}{ 
\sqrt{T} }\sqrt{   \max_{\text{Alg}}\EE_{\text{Alg}}\left[ \sum_{t=0}^{T-1} \sum_{h=0}^{H-1} \min\left\{ \sigma_t^2(s_h^t,a_h^t), 1 \right\}  \right]} + 2 H \sqrt{ \frac{\ln(T^3 H |\Fcal|) }{ N } }    + \frac{2}{T} \\
\end{align*}
where $\text{Alg}$ is any adaptive mapping that maps from history from $t=0$ to the end of the $t-1$ iteration to to some policy $\pi_t$.  This concludes the proof. 
\end{proof}

Below we discuss special cases. 

\subsection{Discrete MDPs}
\label{app:discrete_mdp}

\begin{proposition}[Discrete MDP Bonus] With $\delta \in (0,1)$. With probability at least $1-\delta$, for all $t\in \mathbb{N}$, we have:
\begin{align*}
\left\| \widehat{P}_t(\cdot | s,a) - P^\star(\cdot | s,a) \right\|_{1} \leq \min\left\{  \sqrt{ \frac{S\ln(t^2 SA /\delta)}{N_t(s,a)} }    ,2\right\}.
\end{align*}
\end{proposition}
\begin{proof}
The proof simply uses the concentration result for $\widehat{P}_t$ under the $\ell_1$ norm. For a fixed $t$ and $s,a$ pair, using Lemma 6.2 in \cite{agarwal2019reinforcement}, we have that with probability at least $1-\delta$,
\begin{align*}
\left\| \widehat{P}_t(\cdot | s,a) - P^\star(\cdot | s,a) \right\|_1 \leq \sqrt{ \frac{S \ln(1/\delta)}{N_t(s,a)} }.
\end{align*} Applying union bound over all iterations and all $s,a$ pairs, we conclude the proof. 
\end{proof}

What left is to bound the information gain $\Ical$ for the tabular case. For this, we can simply use the \pref{prop:IG_knr} that we develop in the next section for KNR. This is because in KNR, when we set the feature mapping $\phi(s,a)\in\mathbb{R}^{|\Scal||\Acal|}$ to be a one-hot vector  with zero everywhere except one in the entry corresponding to $(s,a)$ pair, the information gain in KNR is reduced to the information gain in the tabular model. 

\begin{proposition}[Information Gain in discrete MDPs]
We have:
$$\Ical_T = O\left( H S^2 A \cdot \ln(T SA / \delta)\ln(1+TH)  \right).$$
\end{proposition}
\begin{proof}
Using Lemma B.6 in \cite{kakade2020information}, we have:
\begin{align*}
\sum_{t=0}^{T-1} \min\left\{ \sum_{h=0}^{H-1} {\frac{1}{N_t(s_h^t,a_h^t)} },1 \right\} \leq 2 SA \ln\left( 1 + TH  \right).
\end{align*}Now using the definition of information gain, we have:
\begin{align*}
\Ical_{T} & = \sum_{t=0}^{T-1} \sum_{h=0}^{H-1} \min\left\{\sigma^2_t(s_h^t, a_h^t), 1\right\} \leq {S \ln(T^2 SA/\delta)} H \sum_{t=0}^{T-1} \min\left\{ \sum_{h=0}^{H-1} {\frac{1}{N_t(s_h^t,a_h^t)} },1 \right\} \\
& \leq  2 H {S^2 A \ln(T^2 SA/\delta) \ln(1+TH)} 
\end{align*}
This concludes the proof. 
\end{proof}

\subsection{KNRs}
\label{subsec:knr_appendix}
Recall the KNR setting from Example~\ref{exp:knr}. The following proposition shows that the bonus designed in Example~\ref{exp:knr} is valid. 
\begin{proposition}[KNR Bonus] Fix $\delta \in (0,1)$. With probability at least $1-\delta$, for all $t\in\mathbb{N}$, we have:
\begin{align*}
\left\| \widehat{P}_t(\cdot | s,a) - P^\star(\cdot | s,a)  \right\|_{1} \leq \min\left\{ \frac{\beta_t}{\sigma} \left\| \phi(s,a) \right\|_{\Sigma_t^{-1}},2 \right\}, \forall s,a,
\end{align*} where $\beta_t =  \sqrt{ 2\lambda \|W^\star\|^2_2  + 8 \sigma^2 \left(d_s \ln(5) + 2 \ln( t^2 / \delta) + \ln(4) + \ln\left( \det(\Sigma_t) / \det(\lambda I) \right) \right) }$.
\label{prop:knr_bonus}
\end{proposition}
\begin{proof} The proof directly follows the confidence ball construction and proof from \cite{kakade2020information}. Specifically, from Lemma B.5 in \cite{kakade2020information}, we have that with probability at least $1-\delta$, for all $t$:
\begin{align*}
    \left\| \left(\widehat{W}_t  - W^\star\right) \left(\Sigma_t\right)^{1/2}  \right\|_2^2 \leq \beta^2_t. 
\end{align*}
Thus, with \pref{lem:gaussian_tv}, we have:
\begin{align*}
    \left\| \widehat{P}_t(\cdot | s,a) - P^\star(\cdot | s,a)  \right\|_{1} \leq \frac{1}{\sigma} \left\| (\widehat{W}_t - W^\star) \phi(s,a)   \right\|_2 \leq \left\| (\widehat{W}_t - W^\star) (\Sigma_t)^{1/2}  \right\| \left\| \phi(s,a) \right\|_{\Sigma_t^{-1}}  / \sigma \leq \frac{\beta_t}{\sigma} \| \phi(s,a) \|_{\Sigma_t^{-1}}.
\end{align*} This concludes the proof. 
\end{proof}

The following proposition bounds the information gain quantity. 
\begin{proposition}[Information Gain on KNRs] For simplicity, let us assume $\phi: \Scal\times\Acal\mapsto \mathbb{R}^d$, i.e., $\phi(s,a)$ is a d-dim feature vector.  In this case, we will have:
\begin{align*}
    \Ical_{T} = O\left( H \left( d \ln(T^2/\delta) + d d_s + d^2 \ln\left(1+\|W^\star\|_2^2TH/\sigma^2\right) \right)\ln\left(1+\|W^\star\|_2^2TH/\sigma^2\right)   \right).
\end{align*}
\label{prop:IG_knr}
\end{proposition}
\begin{proof}
From the previous proposition, we know that $\sigma^2_t(s,a) = (\beta^2_t/\sigma^2) \|\phi(s,a)\|^2_{\Sigma_t^{-1}}$. Setting $\lambda = \sigma^2 / \|W^\star\|_2^2$, we will have $\beta_t^2 / \sigma^2 \geq 1$, which means that $\min\{\sigma_t^2(s,a) , 1 \} \leq (\beta^2_t/\sigma^2) \min\left\{\left\| \phi(s,a) \right\|^2_{\Sigma_t^{-1}},1\right\}$. 

Note that $\beta_t$ is non-decreasing with respect to $t$, so $\beta_t \leq \beta_T$ for $T \geq t$, where
\begin{align*}
    \beta_T = \sqrt{  2\sigma^2 + 8 \sigma^2 (d_s \ln(5) + 2 \ln(T^2/\delta) + \ln(4) + d \ln( 1 + TH \|W^\star\|_2^2/ \sigma^2 )) }  
\end{align*}
Also we have $\sum_{t=0}^{T-1} \sum_{h=0}^{H-1} \min \left\{ \| \phi(s_h^t, a_h^t) \|^2_{\Sigma_t^{-1}},1\right\} \leq H  \sum_{t=0}^{T-1}  \min \left\{  \sum_{h=0}^{H-1} \| \phi(s_h^t, a_h^t) \|^2_{\Sigma_t^{-1}}, 1\right\} $, since $\min\{a_1, b_1\} + \min\{a_2, b_2\} \leq \min\{a_1 + a_2, b_1 + b_2\}$. Now call Lemma B.6 in \cite{kakade2020information}, we have:
\begin{align}
    \sum_{t=0}^{T-1} \min\left\{ \sum_{h=0}^{H-1} \| \phi(s_h^t,a_h^t) \|^2_{\Sigma_t^{-1}}, 1 \right\} \leq 2 \ln\left( \det(\Sigma_T) / \det(\lambda I) \right) = 2 d \ln\left( 1 + TH \|W^\star\|_2^2 / \sigma^2 \right).
    \label{eq:elliptical_potential}
\end{align}
Finally recall the definition of $\Ical_{T}$, we have:
\begin{align*}
    \Ical_{T} & = \sum_{t=0}^{T-1} \sum_{h=0}^{H-1} \min\left\{\sigma^2_t(s_h^t, a_h^t), 1\right\} \leq \frac{\beta_T^2}{\sigma^2} \sum_{t=0}^{T-1} \sum_{h=0}^{H-1} \min\left\{ \| \phi(s_h^t,a_h^t) \|^2_{\Sigma_t^{-1}},1 \right\} \leq \frac{\beta_T^2}{\sigma^2} 2 H d\ln(1 + \|W^\star\|_2^2 TH / \sigma^2)  \\
    & \leq 2H d\left( 2 + 8 \left(d_s \ln(5) + 2\ln(T^2 / \delta) + \ln(4) + d \ln\left(1+\|W^\star\|_2^2TH/\sigma^2\right)\right) \right) \ln\left(1+\|W^\star\|_2^2TH/\sigma^2\right) \\
    & = H \left( 4d + 32 d d_s + 32 d \ln(T^2 / \delta) + 32 d + 2d^2 \ln\left(1+\|W^\star\|_2^2TH/\sigma^2\right)      \right) \ln\left(1+\|W^\star\|_2^2TH/\sigma^2\right),
\end{align*} which concludes the proof.
\end{proof}

\paragraph{Extension to Infinite Dimensional RKHS}
When $\phi:\mathcal{S}\times\mathcal{A}\mapsto \mathcal{H}$ where $\mathcal{H}$ is some infinite dimensional RKHS, we can bound our regret using the following intrinsic dimension:
\begin{align*}
\widetilde{d} = \max_{ \{\{s_h^t,a_h^t\}_{h=0}^{H-1} \}_{t=0}^{T-1} } \ln\left( I + \frac{1}{\lambda} \sum_{t=0}^{T-1}\sum_{h=0}^{H-1} \phi(s_h^t, a_h^t)\phi(s_h^t, a_h^t)^{\top}  \right).
\end{align*}
In this case, recall \pref{prop:knr_bonus}, we have:
\begin{align*}
\beta_t \leq \beta_{T} & \leq \sqrt{ 2\lambda \|W^\star\|^2_2  + 8 \sigma^2 \left(d_s \ln(5) + 2 \ln( t^2 / \delta) + \ln(4) + \ln\left( \det(\Sigma_T) / \det(\lambda I) \right) \right) } \\
& \leq \sqrt{ 2\lambda \|W^\star\|^2_2  + 8 \sigma^2 \left(d_s \ln(5) + 2 \ln( t^2 / \delta) + \ln(4) +  \widetilde{d}  \right) }.
\end{align*}
Also recall  Eq.~\pref{eq:elliptical_potential}, we have:
\begin{align*}
 \sum_{t=0}^{T-1} \min\left\{ \sum_{h=0}^{H-1} \| \phi(s_h^t,a_h^t) \|^2_{\Sigma_t^{-1}}, 1 \right\} \leq 2 \ln\left( \det(\Sigma_T) / \det(\lambda I) \right) \leq 2\widetilde{d}.
\end{align*}
Combine the above two, following similar derivation we had for finite dimensional setting, we have:
\begin{align*}
\mathcal{I}_{T} = \widetilde{O}\left( H \widetilde{d}^2 + H \widetilde{d} d_s   \right).
\end{align*}
\subsection{General Function Class $\Gcal$ with Bounded Eluder dimension}
\label{subsec:eluder_appendix}
\begin{proposition} Fix $\delta\in (0,1)$. Consider a general function class $\Gcal$ where $\Gcal$ is discrete, and $\sup_{g\in\Gcal,s,a} \|g(s,a)\|_2\leq G$.  At iteration $t$, denote $\widehat{g}_t \in \argmin_{g\in\Gcal}\sum_{i=0}^{t-1}\sum_{h=0}^{H-1} \| g(s_h^i,a_h^i) - s_{h+1}^i \|_2^2$, and denote a version space $\Gcal_t$ as:
\begin{align*}
    \Gcal_t = \left\{ g\in\Gcal: \sum_{i=0}^{t-1}\sum_{h=0}^{H-1} \left\| g(s_h^i,a_h^i) - \widehat{g}_t(s_h^i,a_h^i) \right\|_2^2 \leq c_t \right\}, \text{ with } c_t = 2\sigma^2 G^2 { \ln(2 t^2 |\Gcal|/\delta)  }.
\end{align*}
The with probability at least $1-\delta$, we have that for all $t$, and all $s,a$:
\begin{align*}
    \left\| \widehat{P}_t(\cdot | s,a) -  P^\star (\cdot | s,a) \right\|_1 \leq \min\left\{ \frac{1}{\sigma} \max_{g_1\in\Gcal_t,g_2\in\Gcal_t}  \left\| g_1(s,a) - g_2(s,a) \right\|_2   ,2 \right\}.
\end{align*}\label{prop:uncertainty_eluder}
\end{proposition} 
\begin{proof}
Consider a fixed function $g\in \Gcal$. Let us denote $z^t_{h} = \left\| g(s_h^t,a_h^t) - s_{h+1}^t \right\|_2^2 - \left\| g^\star(s_h^t, a_h^t) - s_{h+1}^t \right\|_2^2$. We have:
\begin{align*}
    z_h^t & = \left( g(s_h^t,a_h^t) - g^\star(s_h^t, a_h^t) \right)^{\top}\left( g(s_h^t,a_h^t) + g^\star(s_h^t, a_h^t)  - 2 g^\star(s_h^t, a_h^t) - 2\epsilon_h^t \right) \\
    & = \left\| g(s_h^t,a_h^t) - g^\star(s_h^t,a_h^t)  \right\|_2^2 - 2 ( g(s_h^t,a_h^t) - g^\star(s_h^t,a_h^t)  )^{\top} \epsilon_h^t.
\end{align*}

Since $\epsilon_h^t \sim \Ncal(0, \sigma^2 I)$, we must have:
\begin{align*}
2 ( g(s_h^t,a_h^t) - g^\star(s_h^t,a_h^t)  )^{\top} \epsilon_h^t \sim \Ncal(0, 4\sigma^2 \left\| g(s_h^t,a_h^t) - g^\star(s_h^t,a_h^t)  \right\|_2^2 )
\end{align*}
Since $\sup_{g,s,a} \|g(s,a)\|_2 \leq G$, then $2 ( g(s_h^t,a_h^t) - g^\star(s_h^t,a_h^t)  )^{\top} \epsilon_h^t$ is a $2\sigma G$ sub-Gaussian random variable. 

Call Lemma 3 in \citep{russo2014learning}, we have that with probability at least $1-\delta$:
\begin{align*}
    \sum_{t}\sum_h \left\| g(s_h^t, a_h^t) - s_{h+1}^t \right\|_2^2 \geq \sum_{t}\sum_h \left\| g^\star(s_h^t,a_h^t) - s_{h+1}^t \right\|_2^2 + 2\sum_t\sum_h \left\| g(s_h^t,a_h^t) - g^\star(s_h^t,a_h^t) \right\|_2^2 - 4\sigma^2 G^2 \ln(1/\delta).
\end{align*}Note that the above can also be derived directly using Azuma-Bernstein's inequality and the property of square loss.
With a union bound over all $g\in \Gcal$, we have that with probability at least $1-\delta$, for all $g\in\Gcal$.
\begin{align*}
\sum_{t}\sum_h \left\| g(s_h^t, a_h^t) - s_{h+1}^t \right\|_2^2 \geq \sum_{t}\sum_h \left\| g^\star(s_h^t,a_h^t) - s_{h+1}^t \right\|_2^2 + 2\sum_t\sum_h \left\| g(s_h^t,a_h^t) - g^\star(s_h^t,a_h^t) \right\|_2^2 - 4\sigma^2 G^2 \ln(|\Gcal|/\delta).
\end{align*}
Set $g = \widehat{g}_t$, and use the fact that $g_t$ is the minimizer of $\sum_t\sum_h \| g(s_h^t,a_h^t) - s_{h+1}^t \|_2^2$, we must have:
\begin{align*}
\sum_t\sum_h \left\| \widehat{g}_t(s_h^t,a_h^t) - g^\star(s_h^t,a_h^t) \right\|_2^2 \leq 2\sigma^2 G^2 {\ln(2t^2 |\Gcal| / \delta) }.
\end{align*}
Namely we prove that our version space $\Gcal_t$ contains $g^\star$ for all $t$.
Thus, we have:
\begin{align*}
\left\| \widehat{P}_t(\cdot | s,a) - P^\star(\cdot|s,a)  \right\|_1 \leq \frac{1}{\sigma} \| \widehat{g}_t(s,a) - g^\star(s,a)  \|_2 \leq \frac{1}{\sigma} \sup_{g_1\in\Gcal_t,g_2\in\Gcal_t} \| g_1(s,a) - g_2(s,a)   \|_2,
\end{align*} where the last inequality holds since both $g^\star$ and $\widehat{g}_t$ belong to the version $\Gcal_t$.

\end{proof}

Now we bound the information gain $\Ical_T$ below.  The proof mainly follows from the proof in \citep{osband2014model}.

\begin{lemma}[Lemma 1 in \cite{osband2014model}]
Denote $\beta_t = 2\sigma^2 G^2 \ln(t^2 |\Gcal| / \delta)$. Let us denote the uncertainty measure $w_{t;h} = \sup_{f_1,f_2\in\Gcal_t} \| f_1(s_h^t, a_h^t) - f_2(s_h^t,a_h^t) \|_2$ (note that $w_{t;h}$ is non-negative). We have:
\begin{align*}
\sum_{i=0}^{t-1}\sum_{h=0}^{H-1} \one\{  w_{t;h}^2 > \epsilon  \} \leq \left( \frac{4\beta_t}{ \epsilon }  + H \right) d_{E}( \sqrt{\epsilon} ).
\end{align*}
\end{lemma}
\begin{proposition}[Bounding $\Ical_T$]  Denote $d = d_E(1 / TH)$. We have
\[\Ical_T = \left( 1/\sigma^2 + HdG^2/\sigma^2 +  8 G^2\ln(T^2 |\Gcal| / \delta) d \ln(TH) \right).\]
\label{prop:eluder_bound_ig}
\end{proposition}
\vspace{-1.5em}
\begin{proof}
Note that the uncertainty measures $w_{t;h}$ are non-negative. Let us reorder the sequence and denote the ordered one as $w_1 \geq w_2 \geq w_3 \dots \geq w_{TH-H}$. For notational simplicity, denote $M = TH - H$
We have:
\begin{align*}
    \sum_{i=0}^{T-1} \sum_{h=0}^{H-1} w_{t;h}^2 = \sum_{i=0}^{M-1} w^2_{i} \leq 1 + \sum_{i} w_{i}^2 \one\{ w_i^2 \geq  \frac{1}{M} \},
\end{align*} where the last inequality comes from the fact that $\sum_{i} w_i^2 \mathbf{1}\{w_i^2 < 1/M\} \leq M \frac{1}{M} = 1$.
Consider any $w_t$ where $w^2_t \geq 1 / M$. In this case, we know that $w^2_1\geq w^2_2 \geq \dots \geq w^2_t \geq 1/M$.  This means that:
\begin{align*}
t \leq \sum_{i}\sum_h \one\{w_{t;h}^2 > w^2_t\} \leq \left( \frac{4\beta_T}{w^2_t} + H \right) d_E(\sqrt{w_t}) \leq \left( \frac{4\beta_T}{w^2_t} + H \right) d_E({1/M}),
\end{align*} where the second inequality uses the lemma above, and the last inequality uses the fact that $d_E(\epsilon)$ is non-decreasing when $\epsilon$ gets smaller. Denote $d = d_E(1/M)$. The above inequality indicates that $w^2_t \leq \frac{4\beta_T d}{ t - Hd }$. This means that for any $w^2_t \geq 1 / M$, we must have $w_t^2 \leq 4 \beta_T d / (t - Hd)$. Thus, we have:
\begin{align*}
\sum_{i=0}^{T-1} \sum_{h=0}^{H-1} w_{t;h}^2 & \leq 1 + Hd G^2 + \sum_{\tau = Hd + 1}^{M} w_\tau^2 \one\{w^2_\tau \geq 1 / M \} \leq  1 + HdG^2 +  4 \beta_T d \ln(M) \\
& =  1 + HdG^2 +  4 \beta_T d \ln(TH).
\end{align*}
Finally, recall the definition of $\Ical_T$, we have:
\begin{align*}
\sum_{t=0}^{T-1} \sum_{h=0}^{H-1} \min\{\sigma_t^2 (s_h^t,a_h^t), 1\} \leq \sum_{t=0}^{T-1}  \sum_{h=0}^{H-1}\sigma_t^2(s_h^t,a_h^t) \leq \frac{1}{\sigma^2} \sum_{t=0}^{T-1} \sum_{h=0}^{H-1} w_{t;h}^2 \leq \frac{1}{\sigma^2} \left( 1 + HdG^2 +  4 \beta_T d \ln(TH) \right).
\end{align*}This concludes the proof. 
\end{proof}

\subsection{Proof of \pref{thm:ILFO_lower_bound}}
\label{app:low_bound}
This section provides the proof of \pref{thm:ILFO_lower_bound}.

First we present the reduction from a bandit optimization problem to ILFO. 

Consider a Multi-armed bandit (MAB) problem with $A$ many actions $\{a_i\}_{i=1}^A$. Each action's ground truth reward $r_i$ is sampled from a Gaussian with mean $\mu_i$ and variance $1$. Without loss of generality, assume $a_1$ is the optimal arm, i.e., $\mu_1 \geq \mu_i \ \forall \ i\neq 1$. We convert this MAB instance into an MDP. Specifically, set $H = 2$. Suppose we have a fixed initial state $s_0$ which has $A$ many actions. For the one step transition, we have $P( \cdot | s_0, a_i ) = \Ncal(\mu_i, 1)$, i.e., $g^*(s_0, a_i) = \mu_i$. Here we denote the optimal expert policy $\pi^e$ as $\pi^e(s_0) = a_1$, i.e., expert policy picks the optimal arm in the MAB instance. Hence, when executing $\pi^e$, we note that the state $s_1$ generated from $\pi^e$ is simply the stochastic reward of $a_1$ in the original MAB instance. Assume that we have observed infinitely many such $s_1$ from the expert policy $\pi^e$, i.e., we have infinitely many samples of expert state data, i.e., $N\to\infty$. Note, however, we do not have the actions taken by the expert (since this is the ILFO setting). This expert data is equivalent to revealing the optimal arm's mean reward $\mu_1$ to the MAB learner a priori.  Hence solving the ILFO problem on this MDP is no easier than solving the original MAB instance with additional information which is that optimal arm's mean reward is $\mu_1$ (but the best arm's identity is unknown). 

Below we show the lower bound for solving the MAB problem where the optimal arm's mean is known. 

\begin{theorem}
Consider best arm identification of Gaussian MAB with the additional information that the optimal arm's mean reward is $\mu$. For any algorithm, there exists a MAB instance with number of arms $A \geq 2$, such that the expected cumulative regret is still $\Omega(\sqrt{AT})$, i.e., the additional information does not help improving the worst-case regret bound to solve the MAB instance. 
    \label{thm:MAB_lower_bound}
\end{theorem}

\begin{proof}[Proof of \pref{thm:MAB_lower_bound}] 
Below, we will construct $A$ many MAB instances where each instance has $A$ many arms and each arm has a Gaussian reward distribution with the fixed variance $\sigma^2$. Each of the $A$ instances has the maximum mean reward equal to $\Delta$, i.e., all these $A$ instances have the same maximum arm mean reward. Consider any algorithm $\mathrm{Alg}$ that maps  $\Delta$ together with the history of the interactions  $\Hcal_t = \{a_0, r_0, a_1,r_1,\dots, a_{t-1}, r_{t-1}\}$ to a distribution over $A$ actions. We will show for any such algorithm $\mathrm{alg}$ that knows $\Delta$, with constant probability, there must exist a MAB instance from the $A$ many MAB instances, such that $\mathrm{Alg}$ suffers at least $\Omega(\sqrt{AT})$ regret where $T$ is the number of iterations. 

Now we construct the $A$ instances as follows. Consider the $i$-th instance ($i = 1,\dots, A$). For arm $j$ in the i-th instance, we define its mean as $\mu^i_j = \mathbf{1}\{ i = j \} \Delta$. Namely, for MAB instance $i$, its arms have mean reward zero everywhere except that the $i$-th arm has reward mean $\Delta$.  Note that all these MAB instances have the same maximum mean reward, i.e., $\Delta$. Hence, we cannot distinguish them by just revealing $\Delta$ to the learner. 

We will construct an additional MAB instance (we name it as $0$-th MAB instance) whose arms have reward mean zero.  Note that this MAB instance has maximum mean reward $0$ which is different from the previous $A$ MAB instances that we constructed. However, we will only look at the regret of $\mathrm{Alg}$ on the previously constructed $A$ MAB instances. I.e., we do not care about the regret of $\mathrm{Alg}(\Delta,\mathcal{H}_t)$ on the $0$-th MAB instance. 

Let us denote $\mathbb{P}_i$ (for $i = 0,\dots, A$) as the distribution of the outcomes of algorithm $\mathrm{Alg}(\Delta, \mathcal{H}_t)$ interacting with MAB instance $i$ for $n$ iterations, and $\EE_j[N_i(T)]$ as the expected number of times arm $i$ is pulled by $\mathrm{Alg}(\Delta, \mathcal{H}_t)$ in MAB instance $j$. Consider MAB instance $i$ with $i \geq 1$:
\begin{align*}
    \mathbb{E}_{i}[N_i(T)]  - \mathbb{E}_{0}[N_i(T)] \leq T \left\| \mathbb{P}_i - \mathbb{P}_{0}   \right\|_{1} \leq T \sqrt{ \mathrm{KL}( \mathbb{P}_0, \mathbb{P}_i) } \leq T \sqrt{ \Delta^2 \mathbb{E}_{0}[N_i(T)] },
\end{align*} where the last step uses the fact that we are running the same algorithm $\text{Alg}(\Delta, \mathcal{H}_t)$ on both instance $0$  and instance $i$ (i.e., same policy for generating actions), and thus, $ \mathrm{KL}( \mathbb{P}_0, \mathbb{P}_i) = \sum_{j=1}^A \mathbb{E}_{0}[N_j(T)] \mathrm{KL}\left( q_0(j), q_{i}(j) \right)$ (Lemma 15.1 in \cite{bandit_alg}), where $q_i(j)$ is the reward distribution of arm $j$ at instance $i$. Also recall that for instance 0 and instance $i$, their rewards only differ at arm $i$. 

\noindent This implies that:
\begin{align*}
    \mathbb{E}_{i}[N_i(T)] \leq \mathbb{E}_{0}[N_i(T)] + T \sqrt{ \Delta^2 \mathbb{E}_{0}[N_i(T)] }.
\end{align*} Sum over $i = 1,\dots, A$ on both sides, we have:
\begin{align*}
    \sum_{i=1}^A \mathbb{E}_{i}[N_i(T)] & \leq T + T \sum_{i=1}^A \sqrt{ \Delta^2 \mathbb{E}_{0}[N_i(T)] } \leq T + T \sqrt{A} \sqrt{\sum_{i=1}^A  \Delta^2  \mathbb{E}_{0}[N_i(T)] } \\
    & \leq T + T \sqrt{A} \sqrt{ \Delta^2 T }
\end{align*}
Now let us calculate the regret of $\mathrm{Alg}(\Delta, \mathcal{H}_t)$ on $i$-th instance, we have:
\begin{equation*}
    R_i  = T \Delta - \mathbb{E}_{i}[N_i(T)] \Delta. 
\end{equation*}Sum over $i = 1,\dots, A$, we have:
\begin{align*}
    \sum_{i=1}^A R_i = \Delta\left( AT - \sum_{i=1}^A \mathbb{E}_{i}[N_i(T)] \right) \geq \Delta\left( AT - T - T \sqrt{A\Delta^2 T} \right)
\end{align*}
    Set $\Delta = c \sqrt{A / T}$ for some $c$ that we will specify later, we get:
\begin{align*} 
    \sum_{i=1}^A R_i  \geq  c \sqrt{\frac{A}{T}}\left( AT - T - c AT \right).
\end{align*}Set $c = 1/4$, we get:
\begin{align*}
    \sum_{i=1}^A R_i  \geq  c \sqrt{\frac{A}{T}}\left( AT - T - c AT \right) \geq \frac{1}{4}  \sqrt{{A}{T}}\left( A - 1 - A/4 \right) = \frac{1}{4}  \sqrt{{A}{T}}\left( 3A/4-1 \right) \geq \frac{1}{4}  \sqrt{{A}{T}}\left( A/4 \right),
\end{align*} assuming $A \geq 2$.

\noindent Thus there must exist $i \in \{1,\dots, A\}$, such that:
\begin{align*}
R_i \geq \frac{1}{16}  \sqrt{{A}{T}}.
\end{align*}
Note that the above construction considered any algorithm $\mathrm{Alg}(\Delta, \mathcal{H}_t)$ that maps $\Delta$ and history to action distributions. Thus it concludes the proof. 
\end{proof}

The hardness result in \pref{thm:MAB_lower_bound} and the reduction from MAB to ILFO together implies the lower bound for ILFO in \pref{thm:ILFO_lower_bound}, namely solving ILFO with cumulative regret smaller then $O(\sqrt{AT})$ will contradict the MAB lower bound in \pref{thm:MAB_lower_bound}.

\section{Auxiliary Lemmas}
\label{subsec:auxiliary_lemmas}
\begin{lemma}[Simulation Lemma] Consider any two functions $f: \Scal\times\Acal \mapsto [0,1]$ and $\widehat{f}:\Scal\times\Acal\mapsto[0,1]$, any two transitions $P$ and $\widehat{P}$, and any policy $\pi:\Scal\mapsto \Delta(\Acal)$. We have:
\begin{align*}
    V^{\pi}_{P; f} - V^{\pi}_{\widehat{P}, \widehat{f}} & = \sum_{h=0}^{H-1} \EE_{s,a\sim d^{\pi}_{P}} \left[ f(s,a) - \widehat{f}(s,a)  + \EE_{s'\sim P(\cdot|s,a)} V^{\pi}_{\widehat{P},\widehat{f}; h}(s') - \EE_{s'\sim \widehat{P}(\cdot|s,a)} V^{\pi}_{\widehat{P},\widehat{f}; h}(s')  \right] \\
    & \leq \sum_{h=0}^{H-1} \EE_{s,a\sim d^{\pi}_{P}} \left[ f(s,a) - \widehat{f}(s,a)  + \|V^{\pi}_{\widehat{P},\widehat{f};h} \|_{\infty} \| P(\cdot | s,a) - \widehat{P}(\cdot | s,a)  \|_1   \right] .
\end{align*} where $V^{\pi}_{P,f;h}$ denotes the value function at time step $h$, under $\pi, P, f$.
\label{lem:simulation}
\end{lemma}
Such simulation lemma is standard in model-based RL literature and can be found, for instance, in the proof of Lemma 10 from \cite{sun2019model}.

\begin{lemma}Consider two Gaussian distribution $P_1 := \Ncal(\mu_1, \sigma^2 I)$ and $P_2 := \Ncal(\mu_2, \sigma^2 I)$. We have:
\begin{align*}
\left\| P_1 - P_2 \right\|_{1} \leq \frac{1}{\sigma} \left\| \mu_1 - \mu_2 \right\|_2.
\end{align*}\label{lem:gaussian_tv}
\end{lemma}
The above lemma can be proved by Pinsker's inequality and the closed-form of the KL divergence between $P_1$ and $P_2$.

\section{Implementation Details}
\label{sec:implementation_details}
\subsection{Environment Setup and Benchmarks}\label{sec:imp_env}
This section sketches the details of how we setup the environments. We utilize the standard environment horizon of $500, 50, 200$ for \cartpole, \reacher, \cartpoles. For \swimmer, \hopper~and \walker, we work with the environment horizon set to $400$~\citep{KurutachCDTA18,nagabandi2018neural,luo2018algorithmic,RajeswaranGameMBRL,MOReL}. Furthermore, for \hopper, \walker, we add the velocity of the center of mass to the state parameterization~\citep{RajeswaranGameMBRL,luo2018algorithmic,MOReL}. As noted in the main text, the expert policy is trained using NPG/TRPO~\citep{Kakade01,SchulmanTRPO} until it hits a value of (approximately) $460, -10, 38, 3000, 2000, 170$ for \cartpole, \reacher, \swimmer, \hopper, \walker, \cartpoles{} respectively. Furthermore, for \walker~we utilized pairs of states $(s,s')$ for defining the feature representation used for parameterizing the discriminator. All the results presented in the experiments section are averaged over five seeds. Furthermore, in terms of baselines, we compare \algname~to BC, BC-O, ILPO, GAIL and GAIFO. Note that BC/GAIL has access to expert actions whereas our algorithm does not have access to the expert actions. We report the average of the best performance offered by BC/BC-O when run with five seeds, even if this occurs at different epochs for each of the runs - this gives an upper hand to BC/BC-O. Moreover, note that for BC, we run the supervised learning algorithm for $500$ passes. Furthermore, we run BC-O/GAIL with same number of online samples as \algname~in order to present our results. Furthermore, we used 2 CPUs with $16$-$32$ GB of RAM usage to perform all our benchmarking runs implemented in Pytorch. Finally, our codebase utilizes Open-AI's implementation of TRPO~\citep{baselines} for environments with discrete actions, and the MJRL repository~\citep{Rajeswaran17nips} for working with continuous action environments. With regards to results in the main paper, our bar graph presenting normalized results was obtained by dividing every algorithm's performance (mean/standard deviation) by the expert mean; for \reacher~because the rewards themselves are negative, we first added a constant offset to make all the algorithm's performance to become positive, then, divided by the mean of expert policy.
\subsection{Practical Implementation of \algname}\label{sec:imp_alg}
\begin{algorithm}[t]
\caption{\algname: Model-based Imitation Learning and Exploring for ILFO (used in practical implementation) }\label{alg:implementation_alg}
\begin{algorithmic}[1]
    \STATE {\bf Require}: Expert Dataset $\Dcal_e$, Access to dynamics of the true environment i.e. $P^\star$.
    \STATE {Initialize} Policy $\pi_{0}$, Discriminator $w_0$, Replay Buffer of pre-determined size $\Dcal$, Dynamics Model $\widehat{P}_{-1}$, Bonus $b_{-1}$.
    \FOR{$t = 0, \cdots,{T-1}$}
        \STATE {\bf Online Interaction}: Execute $\pi_t$ in true environment $P^\star$ to get samples $\Scal_t$. 
        \STATE {\bf Update replay buffer}: $\Dcal=\text{Replay-Buffer-Update}(\Dcal,\Scal_t)$ (refer to section \pref{para:buffer}).
        \STATE {\bf Update dynamics model}:
        Obtain $\widehat{P}_{t}$ by starting at $\widehat{P}_{t-1}$ and update using replay buffer $\Dcal$ (refer to section \pref{para:dyna}). 
        \STATE {\bf Bonus Update}: Update bonus $b_t:\Scal\times\Acal\to\mathbb{R}^+$ using replay buffer $\Dcal$ (refer to section \pref{para:bonus}).
	    \STATE {\bf Discriminator Update}: Update discriminator as $w_{t}\leftarrow\arg\max_{w}L(w;\pi_t,\widehat{P}_t, b_t, \Dcal_e)$ (refer to section \pref{para:disc}).
	    \STATE {\bf Policy Update}: Perform incremental policy update through approximate minimization of $L(\cdot)$, \\\qquad\qquad\qquad\,\, i.e.: $\pi_{t}\leftarrow\arg\min_{\pi}L(\pi;w_{t},\widehat{P}_t, b_t, \Dcal_e)$ by running $K_{PG}$ steps of TRPO (refer to section \pref{para:plan}).
    \ENDFOR
    \STATE  {\bf Return} $\pi_{T}$.
\end{algorithmic}
\end{algorithm}
\noindent We will begin with presenting the implementation details of \algname~(refer to Algorithm~\ref{alg:implementation_alg}):
\subsubsection{Dynamics Model Training}\label{para:dyna} 
As detailed in the main paper, we utilize a class of Gaussian Dynamics Models parameterized by an MLP~\citep{RajeswaranGameMBRL}, i.e. 
$\widehat{P}(s,a):=\mathcal{N}(h_{\theta}(s,a), \sigma^2 I)$, where, $h_\theta(s,a) = s + \sigma_{\Delta_s}\cdot \text{MLP}_\theta(s_c,a_c)$, where, $\theta$ are MLP's trainable parameters, $s_c = (s-\mu_s)/\sigma_s$, $a_c = (a-\mu_a)/\sigma_a$ with $\mu_s,\mu_a$ (and $\sigma_s,\sigma_a$) being the mean of states, actions (and standard deviation of states and actions) in the replay buffer $\mathcal{D}$. Note that we predict normalized state differences instead of the next state directly.

In practice, we fine tune our estimate of dynamics models based on the new contents of the replay buffer as opposed to re-training the models from scratch, which is computationally more expensive.
In particular, we start from the estimate $\widehat{P}_{t-1}$ in the $t-1$ epoch and perform multiple updates gradient updates using the contents of the replay buffer $\Dcal$. We utilize constant stepsize SGD with momentum~\citep{SutskeverMomentum} for updating our dynamics models. Since the distribution of $(s,a,s')$ pairs continually drift as the algorithm progresses (for instance, because we observe a new state), we utilize gradient clipping to ensure our model does not diverge due to the aggressive nature of our updates. 

\subsubsection{Replay Buffer}\label{para:buffer} Since we perform incremental training of our dynamics model, we utilize a replay buffer of a fixed size rather than training our dynamics model on all previously collected online $(s,a,s')$ samples. Note that the replay buffer could contain data from all prior online interactions should we re-train our dynamics model from scratch at every epoch. 

\subsubsection{Design of Bonus Function}\label{para:bonus} We utilize an ensemble of two transition dynamics models incrementally learned using the contents of the replay buffer. Specifically, given the models $h_{\theta_1}(\cdot)$ and $h_{\theta_2}(\cdot)$, we compute the discrepancy as: $\delta(s,a) = ||h_{\theta_1}(s,a)-h_{\theta_2}(s,a)||_2.$ Moreover, given a replay buffer $\Dcal$, we compute the maximum discrepancy as $\delta_{\Dcal} = \max_{(s,a,s')\sim\Dcal} \delta(s,a)$. We then set the bonus as $b(s,a) = \min\left(1, \delta(s,a)/\delta_{\Dcal}\right)\cdot\lambda$, thus ensuring the magnitude of our bonus remains bounded between $[0,\lambda]$ roughly.
\subsubsection{Discriminator Update}\label{para:disc} Recall that $f_w(s) = w^\top\psi(s)$, where $w$ are the parameters of the discriminator. Given a policy $\pi$, the update for the parameters $w$ take the following form:
\begin{align*}
    \max_{w: ||w||^2_2\leq \zeta} L(w;\pi,\widehat{P}, b, \Dcal_e)&:=\mathbb{E}_{(s,a)\sim d^{\pi}_{\widehat{P}}} \left[f_w(s) - b(s,a)\right] - \mathbb{E}_{s\sim\Dcal_e} \left[f_w(s)\right]\\
    \equiv\max_{w} L_\zeta(w;\pi,\widehat{P}, b, \Dcal_e)&=\mathbb{E}_{(s,a)\sim d^{\pi}_{\widehat{P}}} \left[f_w(s) - b(s,a)\right] - \mathbb{E}_{s\sim\Dcal_e} \left[f_w(s)\right] -\frac{1}{2}\cdot\left(||w||_2^2-\zeta\right),\\
    \implies\partial_w L_\zeta(w;\pi,\widehat{P}, b, \Dcal_e) &=\mathbb{E}_{s\sim d^{\pi}_{\widehat{P}}} \left[\psi(s)\right] - \mathbb{E}_{s\sim\Dcal_e} \left[\psi(s)\right]-w\in 0,
\end{align*}
where, $\partial_w L_\zeta(w;\pi,\widehat{P},b,\Dcal_e)$ denotes the sub-differential of $L_\zeta(\cdot)$ wrt $w$. This in particular implies the following:
\begin{enumerate}
    \item {\bf Exact Update:} $w^* = \Pcal_{\Bcal(\zeta)}\left(\mathbb{E}_{s\sim d^{\pi}_{\widehat{P}}} \left[\psi(s)\right] - \mathbb{E}_{s\sim\Dcal_e} \left[\psi(s)\right]\right)$, $\Pcal_{\cdot}$ is the projection operator, and $\Bcal(\zeta)$ is the $\zeta-$norm ball.
    \item {\bf Gradient Ascent Update:} $w_{t+1} = \Pcal_{\Bcal(\zeta)}\left((1-\eta_w) w_{t} + \eta_w \cdot \left(\mathbb{E}_{s\sim d^{\pi}_{\widehat{P}}} \left[\psi(s)\right] - \mathbb{E}_{s\sim\Dcal_e} \left[\psi(s)\right]\right)\right)$, $\eta_w>0$ is the step-size.
\end{enumerate}
We found empirically either of the updates to work reasonably well. In the \swimmer~task, we use the gradient ascent update with $\eta_w = 0.67$, and, in the other tasks, we utilize the exact update. Furthermore, we empirically observe the gradient ascent update to yield more stability compared to the exact updates. In the case of \walker, we found it useful to parameterize the discriminator based on pairs of states $(s,s')$.
\subsubsection{Model-Based Policy Update}\label{para:plan} 
Once the maximization of the discriminator parameters $w$ is performed, consider the policy optimization problem, i.e., 
\begin{align*}
    \min_{\pi}L(\pi;w,\widehat{P}, b, \Dcal_e) &:= \mathbb{E}_{(s,a)\sim d^{\pi}_{\widehat{P}}} \left[f_w(s) - b(s,a)\right] - \mathbb{E}_{s\sim\Dcal_e} \left[f_w(s)\right]\\
    \equiv \min_{\pi}L(\pi;w,\widehat{P}, b, \Dcal_e) &= \mathbb{E}_{(s,a)\sim d^{\pi}_{\widehat{P}}} \left[f_w(s) - b(s,a)\right]
\end{align*} 
Hence we perform model-based policy optimization under $\widehat{P}$ and cost function $f_w(s) - b(s,a)$. In practice, we perform approximate minimization of $L(\cdot)$ by incrementally updating the policy using $K_{PG}$-steps of policy gradient, where, $K_{PG}$ is a tunable hyper-parameter. In our experiments, we find that setting $K_{PG}$ to be around $10$ to generally be a reasonable choice (for precise values, refer to Table \ref{tab:hyper-params}). This paper utilizes TRPO~\citep{SchulmanTRPO} as our choice of policy gradient method; note that this can be replaced by other alternatives including PPO~\citep{SchulmanPPO}, SAC~\citep{HaarnojaSAC} {\em etc.} Similar to practical implementations of existing policy gradient methods, we implement a reward filter by clipping the IPM reward $f(s)$ by truncating it between $c_{\min}$ and $c_{\max}$ as this leads to stability of the policy gradient updates. Note that the minimization is done with access to $\widehat{P}$, which implies we perform {\em model-based} planning. Empirically, for purposes of tuning the exploration-imitation parameter $\lambda$, we minimize a surrogate namely: $\mathbb{E}_{(s,a)\sim d^{\pi}_{\widehat{P}}} \left[(1-\lambda)\cdot f_w(s) - b(s,a)\right]$ (recall that $b(s,a)$ has a factor of $\lambda$ associated with it). This ensures that we can precisely control the magnitude of the bonuses against the IPM costs, which, in our experience is empirically easier to work with.
\subsection{Hyper-parameter Details}\label{sec:imp_hyperparameters}
\begin{center}
\begin{table*}[t]
\begin{adjustbox}{max width=\textwidth}
\begin{tabular}{|l|c|c|c|c|c|c|}
        \toprule
		Parameter & \cartpole & \reacher & \swimmer & \cartpoles & \hopper& \walker\\
		\midrule
		\multicolumn{7}{|l|}{\bf Environment Specifications}\\
        \midrule
        Horizon $H$ & $500$ & $50$ & $400$ & $200$ & $400$ & $400$\\
        Expert Performance ($\approx$) & $460$ & $-10$ & $38$ & $181$&$3000$&$2000$\\
        \# online samples per outer loop &$2\cdot H$ & $2\cdot H$& $2\cdot H$ & $2\cdot H$&$8\cdot H$&$3\cdot H$\\
        \midrule
        \multicolumn{7}{|l|}{\bf Dynamics Model}\\
        \midrule
        Architecture/Non-linearity & MLP($64,64$)/ReLU& MLP($64,64$)/ReLU& MLP($512,512$)/ReLU &MLP($64,64$)/ReLU &MLP($512,512$)/ReLU
        &MLP($512,512$)/ReLU\\
        Optimizer(LR, Momentum, Batch Size) &SGD($0.005,0.99,256$)
        &SGD($0.005,0.99,256$) &SGD($0.005,0.99,256$) &SGD($0.005,0.99,256$) &SGD($0.005,0.99,256$)
        &SGD($0.005,0.99,256$)\\
        \# train passes per outer loop &$20$ &$100$ &$100$ & $20$&$50$&$200$\\
        Grad Clipping &$2.0$ & $2.0$& $1.0$ & $2.0$&$4.0$&$1.0$\\
        Replay Buffer Size & $10\cdot H$& $10\cdot H$& $10\cdot H$ &$10\cdot H$ &$16\cdot H$&$15\cdot H$\\
        \midrule
        \multicolumn{7}{|l|}{\bf Ensemble based bonus}\\
        \midrule
        \# models/bonus range & $2$/$[0,1]$&$2$/$[0,1]$ &$2$/$[0,1]$ &$2$/$[0,1]$ &$2$/$[0,1]$&$2$/$[0,1]$\\
        \midrule
        \midrule
        \multicolumn{7}{|l|}{\bf IPM parameters}\\
        \midrule
        Step size for $w$ update ($\eta_w$) & Exact &Exact &$0.33$ &Exact &Exact&Exact\\
        \# RFFs/BW Heuristic & $128$/$0.1$ quantile & $128$ / $0.1$ quantile & $128$ / $0.1$ quantile & $128$ / $0.1$ quantile &$128$ / $0.1$ quantile &$128$ / $0.1$ quantile\\
        \midrule
        \multicolumn{7}{|l|}{\bf Policy parameterization}\\
        \midrule
        Architecture/Non-linearity &MLP($64,64$)/TanH &MLP($64,64$)/TanH & MLP($64,64$)/TanH &MLP($32,32$)/TanH &MLP($32,32$)/TanH
        &MLP($32,32$)/TanH\\
        Policy Constraints &None &None &None & None&$\log\sigma_{\min}=-1.0$&$\log\sigma_{\min}=-2.0$\\
        \midrule
        \multicolumn{7}{|l|}{\bf Planning Algorithm}\\
        \midrule
        \# model samples per TRPO step & $2\cdot H$ &$10\cdot H$ & $4\cdot H$ & $4\cdot H$&$8\cdot H$&$20\cdot H$ \\
        \# TRPO steps per outer loop ($K_{PG}$) & $3$&$10$ &$20$ & $5$&$10$&$15$\\
        \midrule
        \makecell[l]{TRPO Parameters\\(CG iters, dampening, kl, $\text{gae}_{\lambda}$, $\gamma$)}  &\makecell{$(50,0.001,0.01,$\\$0.97,0.995)$} & \makecell{$(100,0.001,0.01,$\\$0.97,0.995)$}&\makecell{$(100,0.001,0.01,$\\$0.97,0.995)$} &\makecell{$(100,0.001,0.01,$\\$0.97,0.995)$} &\makecell{$(10,0.0001,0.025,$\\$0.97,0.995)$}&\makecell{$(10,0.0001,0.025,$\\$0.97,0.995)$}\\
        \midrule
        \multicolumn{6}{|l|}{\bf Critic parameterization}\\
        \midrule
        Architecture/Non-linearity & MLP($128,128$)/ReLU& MLP($128,128$)/ReLU& MLP($128,128$)/ReLU&MLP($32,32$)/ReLU &MLP($128,128$)/ReLU&MLP($128,128$)/ReLU\\
        \midrule
        \makecell[l]{Optimizer\\(LR, Batch Size, $\epsilon$, Regularization)} & Adam($0.001,64,1e-5,0$)& Adam($0.001,64,1e-5,0$)& Adam($0.001,64,1e-5,0$)& Adam($0.001,64,1e-5,0$)&Adam($0.001,64,1e-8,1e-3$)&Adam($0.001,64,1e-8,1e-3$) \\
        \midrule
        \# train passes per TRPO update &$1$ & $1$&$1$&$1$&$2$&$2$\\
\bottomrule
\end{tabular}
\end{adjustbox}
\caption{List of various Hyper-parameters employed in \algname's implementation.}\label{tab:hyper-params}
\end{table*}
\end{center}
This section presents an overview of the list of hyper-parameters necessary to implement Algorithm~\ref{alg:main_alg_bonus} in practice, as described in Algorithm~\ref{alg:implementation_alg}. The list of hyper-parameters is precisely listed out in Table~\ref{tab:hyper-params}. The hyper-parameters are broadly categorized into ones corresponding to various components of \algname, namely, (a) environment specifications, (b) dynamics model, (c) ensemble based bonus, (d) IPM parameterization, (e) Policy parameterization, (f) Planning algorithm parameters, (g) Critic parameterization. Note that if there a hyper-parameter that has not been listed, for instance, say, the value of momentum for the ADAM optimizer in the critic, this has been left as is the default value defined in Pytorch.
\section{Additional Experimental Results}\label{app:add_expt_results}
\subsection{Modified \cartpoles~environment with noise added to transition dynamics}\label{sec:cartpoles_learning_curves}
\begin{wrapfigure}{r}{0.4\textwidth}
    \vspace{-5mm}
    \centering
    \begin{subfigure}
        \centering
        \includegraphics[width=0.4\textwidth]{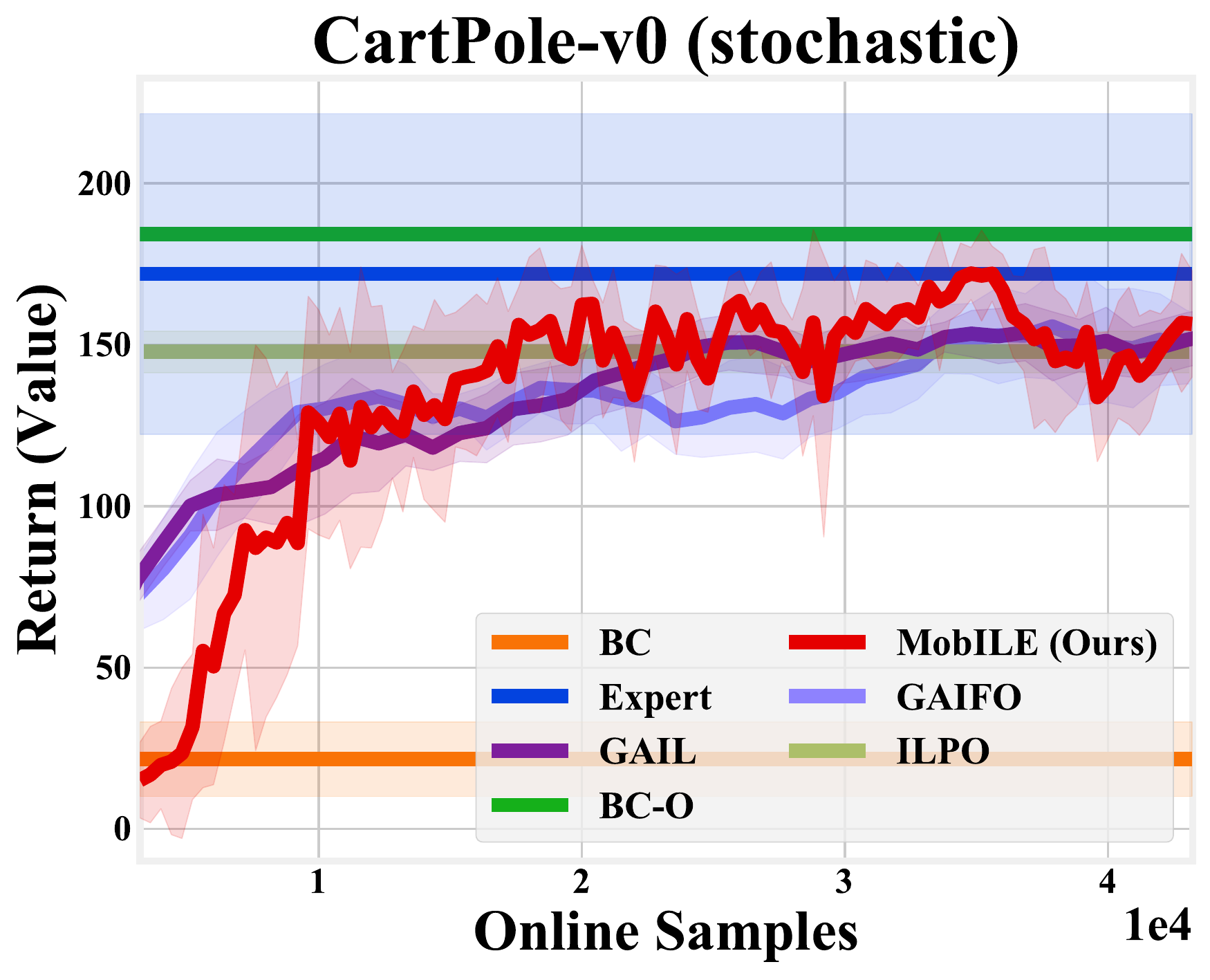}
    \end{subfigure}
    \caption{Learning curves for \cartpoles~with stochastic dynamics with $20$ expert trajectories comparing \algname~with BC, BC-O, GAIL, GAIFO and ILPO. }\label{fig:cartpoles}
    \vspace{-4mm}
\end{wrapfigure}
We consider a stochastic variant of \cartpoles, wherein, we add additive Gaussian noise of variance unknown to the learner in order to make the transition dynamics of the environment to be stochastic. Specifically, we train an expert of value $\approx\ 170$ in \cartpoles{} with stochastic dynamics using TRPO. Now, using $20$ trajectories drawn from this expert, we wish to consider solving the ILFO problem using \algname~as well as other baselines including BC, BC-O, ILPO, GAIL and GAIFO. Figure~\ref{fig:cartpoles} presents the result of this comparison. Note that \algname~compares favorably against other baseline methods - in particular, BC tends suffer in environments like \cartpoles~with stochastic dynamics because of increased generalization error of the supervised learning algorithm used for learning a policy. Our algorithm is competitive with both BC-O, GAIL, GAIFO and ILPO. Note that BC-O tends to outperform BC both in \cartpole~and in \cartpoles~(with stochastic dynamics).
\subsection{Swimmer Learning Curves}\label{sec:swimmer_learning_curves}
We supplement the learning curves for \swimmer~(with 40 expert trajectories) with the learning curves for \swimmer~with 10 expert trajectories in figure~\ref{fig:swimmer}. As can be seen, \algname~outperforms baseline algorithms such as BC, BC-O, ILPO, GAIL and GAIFO in \swimmer~with both $40$ and $10$ expert trajectories. The caveat is that for $10$ expert trajectories, all algorithms tend to show a lot more variance in their behavior and this reduces as we move to the $40$ expert trajectory case.
\begin{figure*}[ht]
    \centering
    \begin{subfigure}
        \centering
        \includegraphics[width=0.6\textwidth]{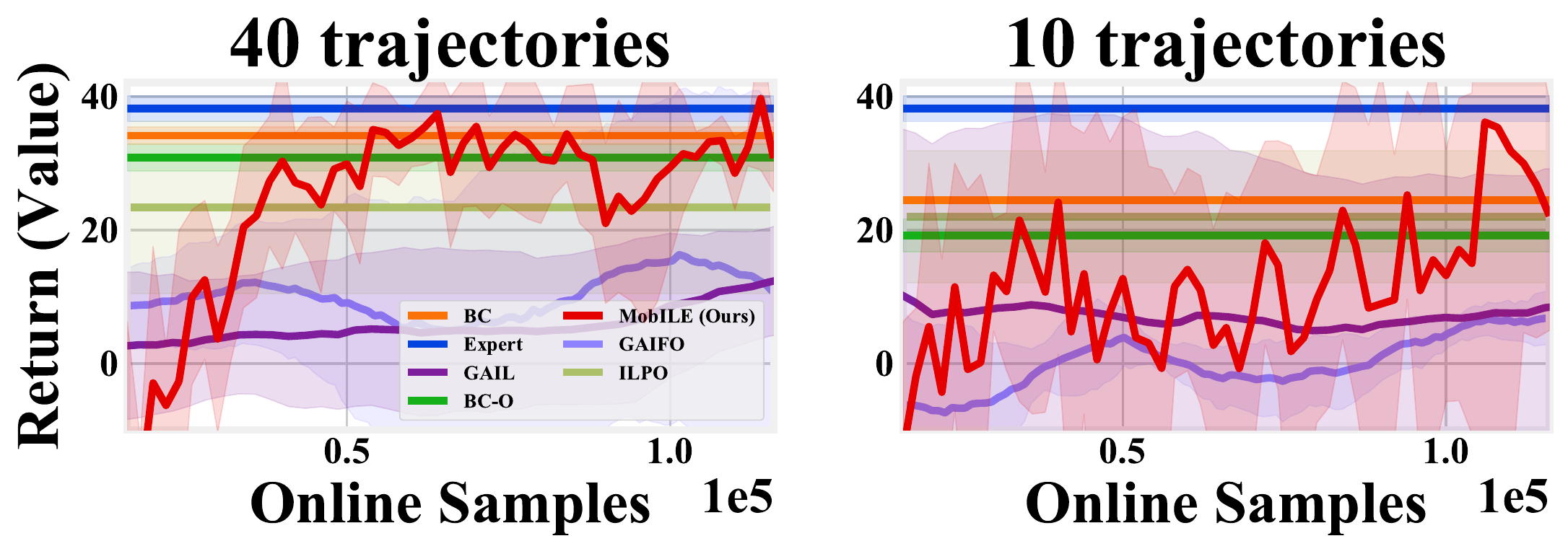}
    \end{subfigure}
    \vspace{-2mm}
    \caption{Learning curves for \swimmer~with $40$ (left) and $10$ (right) expert trajectories comparing \algname~with BC, BC-O, ILPO, GAIL and GAIFO. \algname~continues to perform well relative to all other benchmarks with both $10$ and $40$ expert trajectories. The variance of the algorithm as well as the benchmarks is notably higher with lesser number of expert trajectories.}\label{fig:swimmer}
    \vspace{-4mm}
\end{figure*}
\subsection{Additional Results}\label{sec:max_learning_curves}
\begin{figure*}[ht]
    \centering
    \begin{subfigure}
        \centering
        \includegraphics[width=\textwidth]{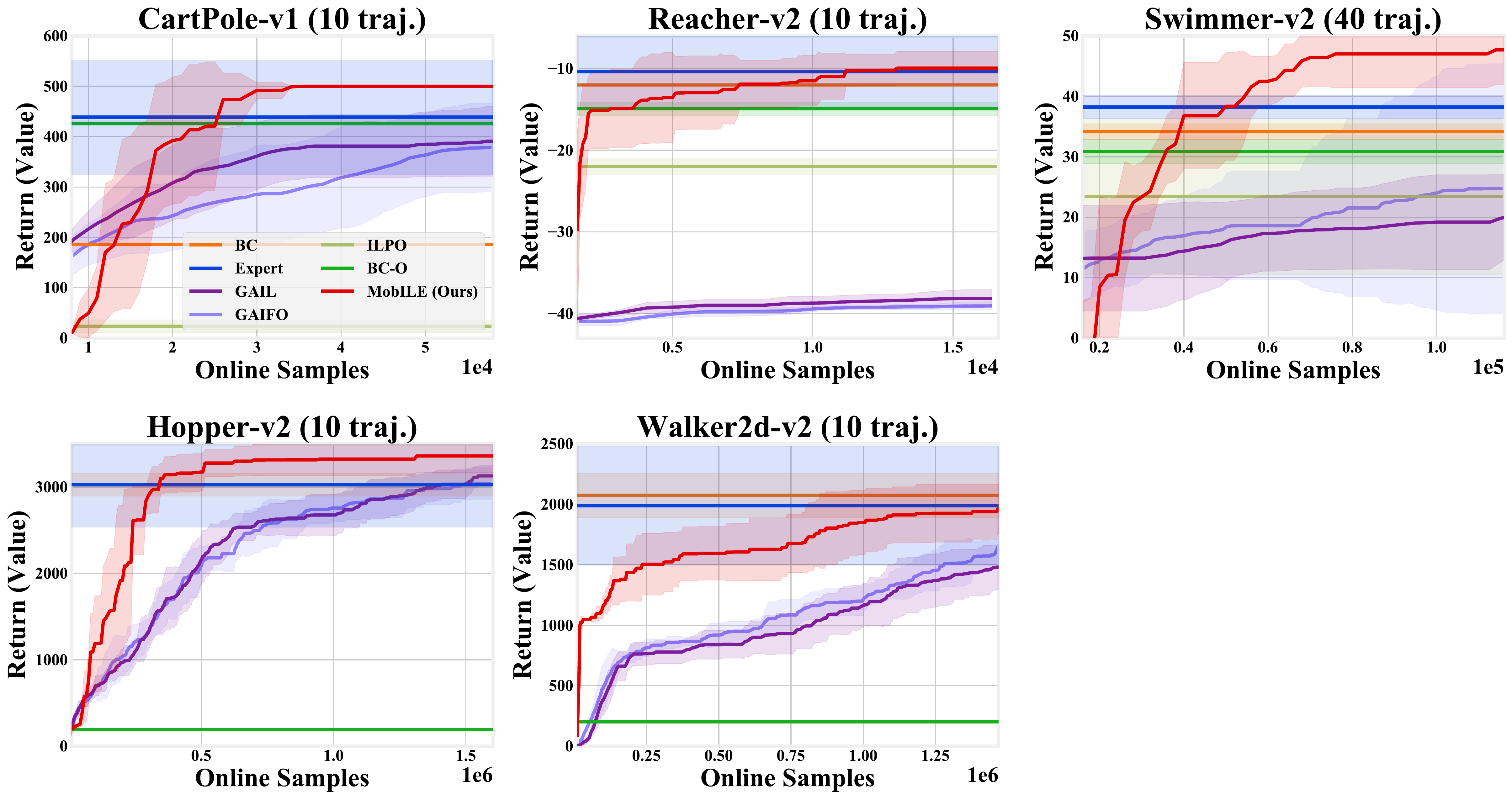}
    \end{subfigure}
    \vspace{-2mm}
    \caption{Learning curves tracking the running maximum averaged across seeds comparing \algname~against BC, BC-O, ILPO, GAIL and GAIFO. \algname~tends to reach expert performance consistently and in a more sample efficient manner.}\label{fig:cumulative_max}
    \vspace{-4mm}
\end{figure*}
In this section, we give another view of our results for \algname~compared against the baselines (BC/BC-O/ILPO/GAIL/GAIFO) by tracking the running maximum of each policy's value averaged across seeds. Specifically, for every iteration $t$, we plot the best policy performance obtained by the algorithm so far averaged across seeds (note that this quantity is monotonic, since the best policy obtained so far can never be worse at a later point of time when running the algorithm). For BC/BC-O/ILPO, we present a simplified view by picking the best policy obtained through the course of running the algorithm and averaging it across seeds (so the curves are flat lines). As figure~\ref{fig:cumulative_max} shows, \algname~reliably hits expert performance faster than GAIL and GAIFO while often matching/outperforming ILPO/BC/BC-O.

\subsection{Ablation Study on Number of Models used for Strategic Exploration Bonus}\label{sec:dynamicsAblation}
In this experiment, we present an ablation study on using more number of models in the ensemble for setting the strategic exploration bonus. Figure~\ref{fig:dynamics-ablation} suggests that even utilizing two models for purposes of setting the bonus is effective from a practical perspective.
\begin{figure*}[ht]
    \centering
    \begin{subfigure}
        \centering
        \includegraphics[width=0.4\textwidth]{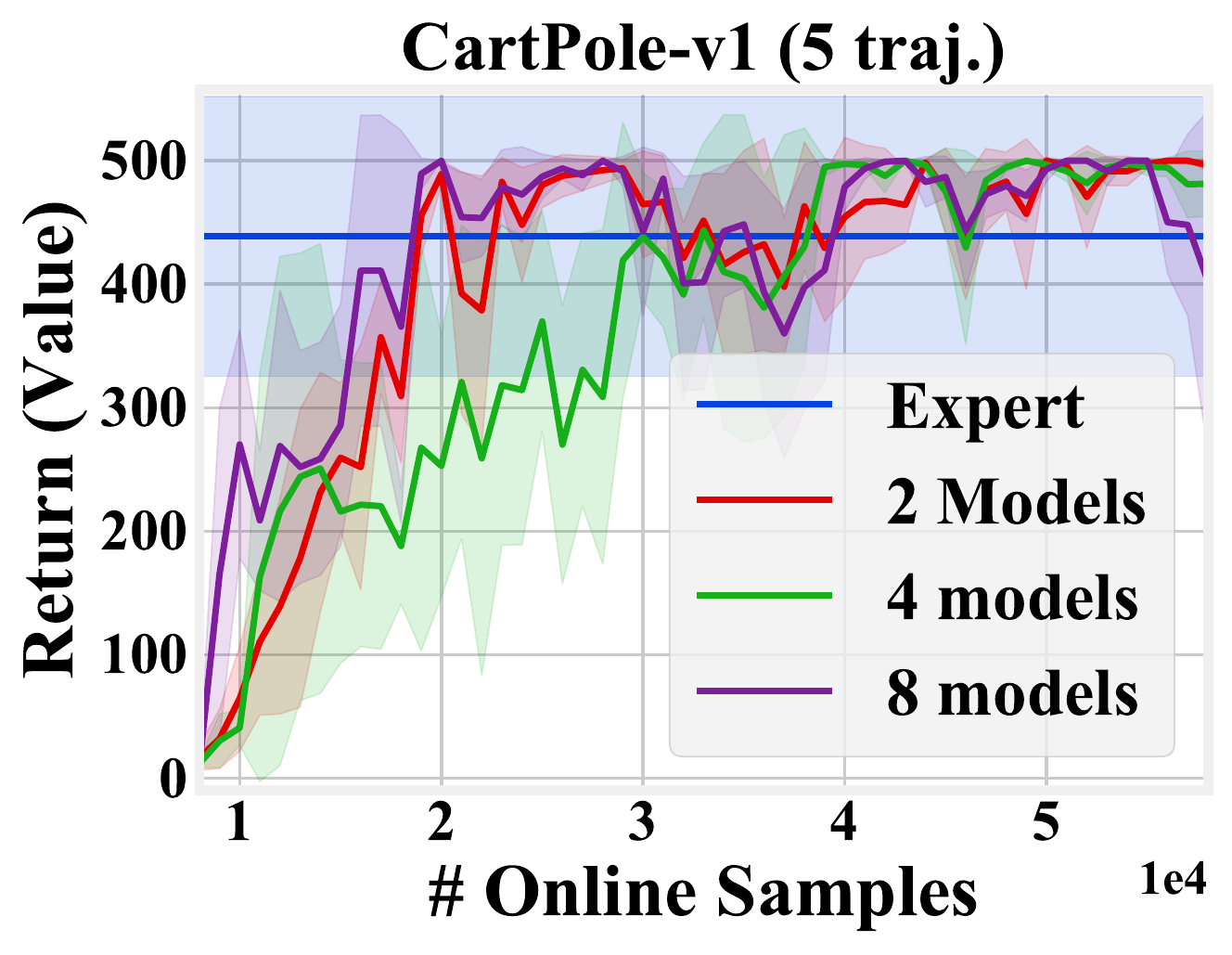}
    \end{subfigure}
    \vspace{-2mm}
    \caption{Learning curves for \cartpole~with varying number of dynamics models for assigning bonuses for strategic exploration.}\label{fig:dynamics-ablation}
    \vspace{-4mm}
\end{figure*}



\end{document}